\documentclass[nohyperref]{article}

\usepackage[dvipsnames]{xcolor}

\usepackage[accepted]{icml2022}

\usepackage{times}
\usepackage[utf8]{inputenc} %
\usepackage[T1]{fontenc}    %
\usepackage{hyperref}       %
\usepackage{url}            %
\usepackage{booktabs}       %
\usepackage{amsfonts}       %
\usepackage{nicefrac}       %
\usepackage{microtype}      %

\usepackage{amsmath}
\usepackage{amssymb}
\usepackage{amsthm}
\usepackage{bm}
\usepackage{dsfont}
\usepackage{mathtools}
\usepackage[capitalize,noabbrev]{cleveref}
\usepackage[group-separator={\,},group-minimum-digits={3}]{siunitx}
\usepackage{subcaption}
\usepackage{thmtools}
\usepackage{thm-restate}
\usepackage{wrapfig}
\usepackage{tikz}
\usetikzlibrary{arrows,intersections,calc,shapes,positioning}

\usepackage[textsize=tiny]{todonotes}

\theoremstyle{plain}

\usepackage{definitions}

\allowdisplaybreaks[1]
\hypersetup{
  colorlinks=true,
  linkcolor={red!50!black},
  citecolor={green!50!black},
  urlcolor={blue!80!black}
}

\icmltitlerunning{On the Surrogate Gap between Contrastive and Supervised Losses}

\begin{document}

\twocolumn[
\icmltitle{On the Surrogate Gap between Contrastive and Supervised Losses}

\icmlsetsymbol{equal}{*}

\begin{icmlauthorlist}
\icmlauthor{Han Bao}{equal,ut,riken}
\icmlauthor{Yoshihiro Nagano}{equal,ut,riken}
\icmlauthor{Kento Nozawa}{equal,ut,riken}
\end{icmlauthorlist}

\icmlaffiliation{ut}{The University of Tokyo, Tokyo, Japan}
\icmlaffiliation{riken}{RIKEN AIP, Tokyo, Japan}

\icmlcorrespondingauthor{Han Bao (currently with Kyoto University)}{bao@i.kyoto-u.ac.jp}

\icmlkeywords{contrastive learning,representation learning}

\vskip 0.3in
]

\printAffiliationsAndNotice{\icmlEqualContribution} %

\begin{abstract}
Contrastive representation learning encourages data representation to make semantically similar pairs closer than randomly drawn negative samples, which has been successful in various domains such as vision, language, and graphs.
Recent theoretical studies have attempted to explain the benefit of the large negative sample size by upper-bounding the downstream classification loss with the contrastive loss.
However, the previous surrogate bounds have two drawbacks: they are only legitimate for a limited range of negative sample sizes and prohibitively large even within that range.
Due to these drawbacks, there still does not exist a consensus on \emph{how negative sample size theoretically correlates with downstream classification performance}.
Following the simplified setting where positive pairs are drawn from the true distribution (not generated by data augmentation; as supposed in previous studies),
this study establishes surrogate upper and lower bounds for the downstream classification loss for all negative sample sizes that best explain the empirical observations on the negative sample size in the earlier studies.
Our bounds suggest that the contrastive loss can be viewed as a surrogate objective of the downstream loss and larger negative sample sizes improve downstream classification because the surrogate gap between contrastive and supervised losses decays.
We verify that our theory is consistent with experiments on synthetic, vision, and language datasets.
\end{abstract}

\section{Introduction}

\begin{figure}[t]
  \centering
  \includegraphics[width=0.85\linewidth]{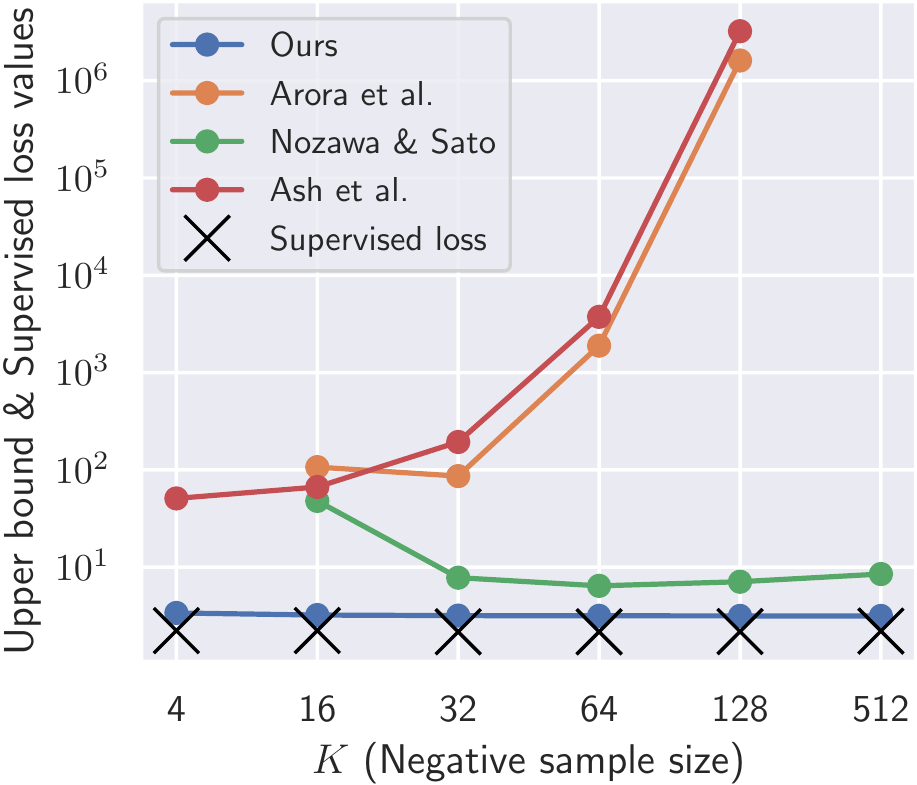}
  \caption{
    Empirical comparison of our upper bound and the existing bounds of the classification loss with CIFAR-10 ($C = \num{10}$),
    showing that \emph{we improve the classification loss bound by exponentially large margins}.
    \citeauthor{Arora2019ICML}'s and \citeauthor{Nozawa2021NeurIPS}'s bounds are valid only at $K + 1 \ge C$.
    Note that \citeauthor{Arora2019ICML}'s and \citeauthor{Ash2022AISTATS}'s bounds become infinity at $K = \num{512}$,
    where $K$ and $C$ are the numbers of negative samples and classes, respectively.
    As can be seen, our bound is the closest approximator of the true mean supervised losses.
    The detailed setup is in \cref{section:large_experiments}.
  }
  \label{figure:upper_bound_comparison_cifar10}
\end{figure}

The contrastive loss~\citep{Chopra2005CVPR} is one of the popular loss functions in metric learning~\citep{Kulis2012} and representation learning~\citep{Bengio2013IEEE}.
The contrastive loss forces data representation of semantically similar pairs closer in some metric space than multiple random samples, called negative samples.
Many state-of-the-art representation learning algorithms use a type of contrastive losses
in natural language processing~\citep{Mikolov2013NeurIPS,Logeswaran2018ICLR},
vision~\citep{Chopra2005CVPR,He2019CVPR,Chen2020ICML},
and graph~\citep{Kirong2021arXiv} domains.
A simple model built on top of the learned representation can achieve almost the same accuracy as supervised learning does.

Recent empirical studies observed that downstream classification performance could be improved with a sufficiently large negative sample size (denoted by $K$), compared with the number of classes (denoted by $C$)~\citep{He2020CVPR,Chen2020ICML}.
To better understand the underlying mechanism of this large $K$ benefit, several studies attempted to derive surrogate upper bounds of the downstream classification loss by the contrastive loss.
\citet{Arora2019ICML} successfully established the first upper bound, which exponentially deteriorates with larger $K$ because the labels of negative samples frequently collide with the positive sample (called \emph{label collision}), contradicting the larger-$K$ benefit.
By contrast, \citet{Nozawa2021NeurIPS} argued that supervised classes could be \emph{covered} by negative samples with higher probability as $K$ becomes larger (called \emph{label coverage}), supported by their bound.
While their claim agrees with the large-$K$ benefit, their bound holds only when $K > C+1$, and hence does not explain the empirical observation that contrastive learning works to some extent even with small $K$~\citep{Chen2021NeurIPS,Tomasev2022arXiv}.
Furthermore, \citet{Ash2022AISTATS} advocated the existence of \emph{collision-coverage trade-off} so that their upper bound has an \emph{optimal} $K$.
Why have we yet to reach a consensus?
We observe that the lack of a consensus is due to the existing upper bounds having the following drawbacks: some bounds are only valid within a limited range of $K$, and prohibitively large even within that range.
\cref{figure:upper_bound_comparison_cifar10} shows the comparison of the existing bounds and the classification loss.
Hence, we ask the following research question: \emph{How does negative sample size $K$ affect the downstream classification performance?}

In this study, we derive a surrogate gap bound of the downstream classification loss that is applicable to any $K$ and shrinks with larger $K$.
In particular, we derive not only the upper (\cref{theorem:curl_upper_bound}) but also the \textit{lower} bound (\cref{theorem:curl_lower_bound}) of the downstream loss, and show the tightness in $K$.
As the gap between upper and lower bounds shrinks in $O(K^{-1})$, the contrastive loss can be viewed as a surrogate objective of the downstream classification loss, and the downstream performance is improved by larger $K$.
This is consistent with the empirical observations that larger $K$ improves the downstream performance whereas contrastive learning can work to some extent even with small $K$ (\cref{section:discussion}).
In addition, our bounds resolve the controversy among the existing bounds so that there is no collision-coverage trade-off in $K$ and the large-$K$ benefit is witnessed (\cref{section:comparison_with_existing}).
Finally, we empirically verify our theory by experiments (\cref{section:experiments}) on a synthetic dataset, CIFAR-10/100~\citep{Krizhevsky2009techrep} datasets, and Wiki-3029 dataset~\citep{Arora2019ICML}.
Note that we assume that positive pairs are drawn from the true underlying distribution instead of generated by data augmentation for simplicity, as supposed in the previous studies.

\section{Formulation of Contrastive Learning}
\label{section:formulation}
First, this section briefly summarizes the problem setup and formulation of contrastive unsupervised representation learning (CURL).%
\footnote{
  We refer to our problem setting as contrastive \emph{unsupervised} representation learning by following \citet{Arora2019ICML} while CURL is provided with the contrastive supervision.
}

\paragraph{Notation.}
The $C$-dimensional vector whose elements are all ones is denoted by $\onebf_C \eqdef [1 \; 1 \; \dots \; 1]^\top$.
When it is clear from context, the subscript is abbreviated.
For a vector $\abf \in \Rbb^p$, $a_{(i)}$ denotes the $i$-th largest element of $\abf$,
namely, $a_{(1)} \ge a_{(2)} \ge \dots \ge a_{(p)}$.
Likewise, $a_{(-i)}$ denotes the $i$-th smallest element of the vector $\abf$.
The indicator function is denoted by $\indicator{A}$ for a predicate $A$.
Let $\triangle^C \eqdef \{\pbf \in [0, 1]^C \mid \pbf^\top\onebf = 1\}$ be the $C$-dimensional probability simplex.
For $\pbf \in \triangle^C$, the Shannon entropy is denoted by $\entropy{\pbf} \eqdef -\sum_{c \in [C]} p_c \ln p_c$.

\paragraph{Supervised classification.}
One of the goals in machine learning is supervised classification,
while we consider the setup where the label supervision is unavailable.
Here, we first formulate $C$-class classification problem for $C \in \Nbb$.
Let $\Xcal$ be $d$-dimensional feature space and $\Ycal \eqdef [C]$ be the supervised label set.
In the supervised setup, we are interested in the following risk quantity, the \emph{supervised loss}, for a multi-class classifier $\gbf: \Xcal \to \Rbb^C$:
\begin{align}
  \suploss(\gbf) \eqdef \E_{\xbf, y \sim \Pbb}\left[ -\ln\frac{e^{g_y(\xbf)}}{\sum_{c \in \Ycal} e^{g_c(\xbf)}} \right],
\end{align}
which is specialized for the softmax cross-entropy loss.
The expectation is taken over the unknown underlying joint distribution $\Pbb$.
Test prediction is given by $\argmax_{y \in \Ycal} g_y(\cdot)$.

\paragraph{Contrastive unsupervised representation learning.}
In the CURL framework~\citep{Arora2019ICML}, we target to learn meaningful data representation by training a similarity model to make the representation of \emph{positive} pairs more similar than randomly drawn $K$ \emph{negative} samples.
The class-conditional distribution is denoted by $\Dcal_c \eqdef \Pbb(X \mid Y = c)$ for each $c \in \Ycal$
and the class-prior distribution by $\pibf \eqdef [\Pbb(Y = c)]_{c \in \Ycal} \in \triangle^C$.
The data generating process is described as follows:
(i)~draw positive/negative classes: $c^+$, $\{c^-_k\}_{k=1}^K \sim \pibf^{K+1}$
(ii)~draw an anchor sample $\xbf \sim \Dcal_{c^+}$
(iii)~draw a positive sample $\xbf^+ \sim \Dcal_{c^+}$%
(iv)~draw $K$ negative samples $\xbf^-_k \sim \Dcal_{c^-_k}$ (for each $k \in [K]$).

In CURL, the representation is learned through minimization of the following \emph{contrastive loss} $\nceloss(\fbf)$
\begin{align}
  \E_{\substack{c^+, \{c^-_k\} \\ \xbf, \xbf^+, \{\xbf^-_k\}}} \!\!\! \biggl[
    -\ln\frac{e^{\fbf(\xbf)^\top\fbf(\xbf^+)}}{e^{\fbf(\xbf)^\top\fbf(\xbf^+)} + \sum_{k \in [K]} e^{\fbf(\xbf)^\top\fbf(\xbf^-_k)}}
  \biggr].
  \label{eq:contrastive_loss}
\end{align}

\paragraph{Evaluation of representations.}
Now we specify our model of classifiers to evaluate learned representations.
A multi-class classifier $\gbf: \Xcal \to \Rbb^C$ consists of learned representation $\fbf: \Xcal \to \Rbb^h$ (frozen) and linear parameters $\Wbf \in \Rbb^{C \times h}$ as
$\gbf(\cdot) \eqdef \Wbf \fbf(\cdot)$,
where $h \in \Nbb$ denotes the dimensionality of the representation given in advance.

For the sake of evaluation, a specific linear classifier called \emph{mean classifier} is introduced.
Given representation $\fbf$, the mean classifier $\Wbf^\mu$ is defined as $\Wbf^\mu \eqdef \left[ \mubf_1 \cdots \mubf_C \right]^\top$, where $\mubf_c \eqdef \E_{\xbf \sim \Dcal_c}[\fbf(\xbf)]$.
This will later be used for evaluating the representation $\fbf$ combined with the supervised loss,
which is denoted by $\msuploss(\fbf) \eqdef \suploss(\Wbf^\mu\fbf)$.
We call it the \emph{mean supervised loss}.
If the mean supervised loss is successfully bounded from above, we end up a bound on the supervised loss through
$\inf_{\Wbf \in \Rbb^{C \times h}} \suploss(\Wbf \fbf) \leq \msuploss(\fbf)$.
For this reason, an upper bound on $\msuploss$ is an intermediate milestone that we seek in this paper.

\section{Surrogate Bounds for Contrastive Learning}
\label{section:learning_bounds}

In this section, our main theoretical results are provided.
We aim at showing that the contrastive loss $\nceloss(\fbf)$ serves as a good estimator of the mean supervised loss $\msuploss(\fbf)$ for any $\fbf$.
We show this by establishing upper and lower bounds of $\msuploss(\fbf)$ by $\nceloss(\fbf)$.
Eventually, the minimization of $\nceloss(\fbf)$ may lead to a good minimizer of $\msuploss(\fbf)$.
All proofs are provided in \cref{appendix:proofs}.

\subsection{Assumptions}
\label{section:assumtions}

Before proceeding with the main results, we explicitly state assumptions used throughout this paper and discuss their validity.

\paragraph{Conditional independence.}
In \cref{section:formulation}, we assumed that anchor and positive samples are conditionally independent: $\xbf \indep \xbf^+ \mid c^+$,
whereas data augmentation (DA) is commonly combined with contrastive learning, and the assumption no longer holds.
While a concurrent work~\citep{Wang2022ICLR} attempted to mitigate this assumption, we work with this assumption to concentrate on the theoretical relationship between $K$ and the downstream performance.
With this assumption, it is possible to compare our result with the previous bounds in a relatively fair manner since the previous studies assumed the same assumption~\citep{Arora2019ICML,Nozawa2021NeurIPS,Ash2022AISTATS}.
The conditional independence assumption has been used in metric learning~\citep{Bellet2012ICML} and weakly-supervised learning~\citep{Bao2018ICML}.
In \cref{appendix:relaxation}, we discuss how to relax the conditional independence assumption.

\paragraph{Existence of supervised classes.}
In unsupervised representation learning, the latent classes $[C]$ and the downstream supervised classes $\Ycal$ are often distinguished.
To draw the connection between learned representation and downstream classification, we must suppose the relationship between $[C]$ and $\Ycal$.
We assume $\Ycal = [C]$ for ease of exposition in the main part.
This assumption can be relaxed to some extent (similarly to \citet{Arora2019ICML}), which will be discussed in \cref{appendix:relaxation}.

\paragraph{Bounded feature representation.}
The size of the representation $\|\fbf(\xbf)\|_2$ is assumed to be bounded.
This assumption is reasonable from the experimental perspective since it is common to normalize representation to employ the cosine similarity as the similarity metric.
Several works reported that the normalized embeddings improve the performance~\citep{Chen2020ICML,Wang2020ICML}.
The existing theoretical work~\citep{Arora2019ICML} also assumes the bounded feature extractor.
Unlike the existing analyses (reviewed in \cref{section:comparison_with_existing}), we take advantage of this assumption to derive the sharp bounds.

\paragraph{Loss function.}
We focus on the cross-entropy-type contrastive loss \eqref{eq:contrastive_loss} because this is the most commonly used loss function~\citep{Mikolov2013NeurIPS,Logeswaran2018ICLR,Chen2020ICML} and its connection to mutual information has been actively discussed~\citep{Tian2020NeurIPS,Tschannen2020ICLR}, while other contrastive loss functions have been proposed in a few recent studies~\citep{Li2021NeurIPS,HaoChen2021NeurIPS,Chuang2022arXiv}.

\subsection{Main Results}
\label{section:main_results}

Below, we investigate the \emph{surrogate gap} $\msuploss(\fbf) - \nceloss(\fbf)$ for a fixed representation $\fbf$.
If the surrogate gap is bounded sufficiently small, the contrastive loss $\nceloss(\fbf)$ can be regarded a good surrogate objective for $\msuploss(\fbf)$.

First, we show a sharp upper bound of the mean supervised loss.
Unlike the existing surrogate bounds of CURL, the upper bound obtained here has a constant coefficient in the contrastive loss and is applicable for all $C$ and $K$ (see discussions in \cref{section:comparison_with_existing}).%

\begin{restatable}{theorem}{thmcurlupperbound}
  \label{theorem:curl_upper_bound}
  For all $\fbf$ such that $\|\fbf(\xbf)\|_2 \le L$ ($\forall \xbf \in \Xcal$), the following inequality holds.
  \begin{align}
    \msuploss(\fbf) \le \nceloss(\fbf) + \Delta_\mathrm{U},
  \end{align}
  where
  $\Delta_\mathrm{U} \eqdef \ln\{\pi_{(1)}K^{-1}C^2\cosh^2(L^2)\}$.
\end{restatable}

Next, the lower bound of the mean supervised loss is provided.
While the existing theoretical analyses often provided upper bounds with a huge coefficient in the contrastive loss,
our lower bound provided below has the same constant coefficient and intercept ($\Delta_\mathrm{U}$ and $\Delta_\mathrm{L}$) rate as our upper bound, ensuring the tightness of our analysis.

\begin{restatable}{theorem}{thmcurllowerbound}
  \label{theorem:curl_lower_bound}
  For all $\fbf$ such that $\|\fbf(\xbf)\|_2 \le L$ ($\forall \xbf \in \Xcal$), the following inequality holds.
  \begin{align}
    \msuploss(\fbf) \ge \nceloss(\fbf) + \Delta_\mathrm{L},
  \end{align}
  where
  $\Delta_\mathrm{L} \eqdef \entropy{\pibf} + \ln\frac{K}{(K + 1)^2} - 2\ln\cosh(L^2)$.
\end{restatable}

Our proofs leverage that the contrastive loss and mean supervised loss share the similar log-sum-exp functional form and directly apply the Jensen's inequality.
This is in contrast to the existing works including \citet{Arora2019ICML}, which approximate the mean supervised loss with the contrastive loss by taking the expectation over latent classes, leading to an exponentially large coefficient.

As we see in \cref{section:discussion}, $\Delta_\mathrm{U}$ and $\Delta_\mathrm{L}$ are the same order in $K$ under the uniform class prior assumption.
By applying either the high-probability bound~\citep{Arora2019ICML} or PAC-Bayesian analysis~\citep{Nozawa2020UAI}, \cref{theorem:curl_upper_bound} (\cref{theorem:curl_lower_bound} as well) can be naturally extended to the form $\msuploss(\widehat\fbf) \le \nceloss(\fbf) + \Delta_\mathrm{U} + \chi$ with a complexity term $\chi$, where $\widehat\fbf$ is the empirical minimizer of the contrastive loss.
Since this is a routine and does not affect the surrogate gap, we omit the high-probability bounds.

\subsection{Discussion}
\label{section:discussion}

\begin{figure}
  \centering
  \resizebox{\linewidth}{!}{%
    \def\ncelossastvalue{2.15072915}
\def\msuplossastvalue{2.0322750988765597}
\def\ubintercept{0.8675616609660546}
\def\lbintercept{-0.9071669155584147}
\def\offsetvalue{0.3}
\def\voffsetvalue{0.1}
\def\xrangevalue{8}
\def\yrangevalue{5}

\begin{tikzpicture}[
    thick,
    >=stealth',
    dot/.style = {
      draw,
      fill = black,
      circle,
      inner sep = 0pt,
      minimum size = 4pt
    },
    font=\LARGE
  ]
  \coordinate (O) at (0,0);
  \coordinate (xunit) at (1,0);
  \coordinate (yunit) at (0,1);
  \coordinate (xrange) at ($\xrangevalue*(xunit)$);
  \coordinate (yrange) at ($\yrangevalue*(yunit)$);
  \coordinate (xoffset) at ($\offsetvalue*(xunit)$);
  \coordinate (yoffset) at ($\offsetvalue*(yunit)$);
  \coordinate (msupupperbound) at ($(\ncelossastvalue,\ncelossastvalue+\ubintercept)$);
  \coordinate (msuplowerbound) at ($(\msuplossastvalue-\lbintercept,\msuplossastvalue)$);
  \coordinate (approachablelimit) at ($(\ncelossastvalue,\msuplossastvalue)$);
  \coordinate (lbmax) at ($(\yrangevalue-\lbintercept,\yrangevalue)$);
  \coordinate (ubmax) at ($(\yrangevalue-\ubintercept,\yrangevalue)$);
  \coordinate (bestpoint) at ($(\ncelossastvalue,\msuplossastvalue)$);

  \draw[->] ($(O)-(xoffset)$) -- ($(xrange)+(xoffset)+(1.9,0)$) coordinate[label = {below:$\nceloss (\fbf)$}] (xmax);
  \draw[->] ($(O)-(yoffset)$) -- ($(yrange)+(yoffset)$) coordinate[label = {left:$\msuploss (\fbf)$}] (xmax);

  \draw[dotted] ($(-\offsetvalue,\msuplossastvalue)$) -- ($(\xrangevalue,\msuplossastvalue)+(1.9,0)$);
  \draw ($(-\offsetvalue,\msuplossastvalue)$) node[left] {$\msuplossast$};

  \draw[dotted] ($(\ncelossastvalue,-\offsetvalue)$) -- ($(\ncelossastvalue,\yrangevalue)$);
  \draw ($(\ncelossastvalue,-\offsetvalue)$) node[below] {$\ncelossast$};

  \draw[dashdotted] ($(-\offsetvalue-\lbintercept,-\offsetvalue)$) -- (lbmax);
  \node[below right=1 and -0.8 of lbmax] {$\msuploss (\fbf) \ge \nceloss (\fbf) + \Delta_\mathrm{L}$};
  \draw[->, out=220, in=320] ([xshift=-10pt, yshift=-55pt]lbmax) to ([xshift=-50pt, yshift=-55pt]lbmax);

  \draw ($(-\offsetvalue,-\offsetvalue+\ubintercept)$) -- (ubmax);
  \node[above=0.8 of ubmax] {$\msuploss (\fbf) \le \nceloss (\fbf) + \Delta_\mathrm{U}$};
  \draw[->, out=250, in=110] ([yshift=25pt]ubmax) to ([yshift=2pt]ubmax);

  \draw[dotted] ($(0,\msuplossastvalue+\ubintercept+\voffsetvalue)$) -- ($(\ncelossastvalue,\msuplossastvalue+\ubintercept+\voffsetvalue)$);
  \draw ($(0,\msuplossastvalue+\ubintercept+\voffsetvalue)$) node[left] {$\ncelossast + \Delta_{\mathrm{U}}$};

  \draw[dotted] ($(\ncelossastvalue-\lbintercept-\voffsetvalue,0)$) -- ($(\ncelossastvalue-\lbintercept-\voffsetvalue,\msuplossastvalue)$);
  \draw ($(\ncelossastvalue-\lbintercept-0.4,0)$) node[below right] {$\msuplossast - \Delta_{\mathrm{L}}$};

  \fill (msupupperbound) circle (0.13);
  \node[draw,rectangle,fill=black,inner sep=3pt] at (msuplowerbound) {};
  \node[draw,star,star points=5,star point ratio=2.25,fill=black,inner sep=1.5pt] at (bestpoint) {};

  \fill[opacity=.1] (approachablelimit) -- (msuplowerbound) -- (lbmax) -- (ubmax) -- (msupupperbound);
\end{tikzpicture}
  }
  \caption{
    The surrogate bounds and feasible region.
    The point \mystar, $(\ncelossast,\msuplossast)$, is the optimal point in the feasible region.
    The points \mybullet and \mysquare are mentioned in the texts.
  }
  \label{figure:feasible_region}
\end{figure}

Subsequently, we discuss implications of our main results on the relationship between the mean supervised loss and $K$.
For the sake of simplicity, we assume $\pi_c = \nicefrac{1}{C}$ for all $c \in [C]$ (the uniform class prior) in this section.

\paragraph{Gap between upper and lower bounds.}
Both of our upper (\cref{theorem:curl_upper_bound}) and lower (\cref{theorem:curl_lower_bound}) bounds draw the linear relationship between the mean supervised loss $\msuploss$ and the contrastive loss $\nceloss$,
with the additional intercept terms $\Delta_\mathrm{U}$ and $\Delta_\mathrm{L}$.
Under the uniform class prior assumption, the intercepts are in the same order:
\begin{align*}
  \Delta_\mathrm{U} &= \ln\left(\nicefrac{C}{K}\right) + 2\ln\cosh(L^2) = O \left(\ln\left(\nicefrac{1}{K}\right)\right), \\
  \Delta_\mathrm{L} &= \ln\left(\nicefrac{CK}{(K+1)^2}\right) - 2\ln\cosh(L^2) = O \left(\ln\left(\nicefrac{1}{K}\right)\right),
\end{align*}
and the gap between two bounds $\Delta_\mathrm{U} - \Delta_\mathrm{L}$ is%
\footnote{
  The approximation $\ln(1 + z) \approx z$ is used (for $0 < z \ll 1$).
}
\begin{align}
  4\ln\cosh(L^2) + 2\ln\left(1 + \nicefrac{1}{K}\right) = O\left(K^{-1}\right),
\end{align}
meaning that the gap shrinks to $4\ln\cosh(L^2)$ as $K$ increases.
Hence, our bounds have the tight intercepts, and \emph{the larger $K$ is beneficial for CURL from the viewpoint of the surrogate gap of the mean supervised loss.}

\paragraph{Surrogate bounds and feasible region.}

Next, we consider the $(\nceloss, \msuploss)$-plot, in which a point indicates $(\nceloss(\fbf), \msuploss(\fbf))$ for some $\fbf$ (see \cref{figure:feasible_region}).
Here, let us focus on the feasible region in the $(\nceloss, \msuploss)$-plot by assuming $\|\fbf\|_2 \le L$ for any $\fbf$ (same as \cref{theorem:curl_upper_bound,theorem:curl_lower_bound}).
Then, the mean supervised loss and contrastive loss are essentially lower-bounded by the constants%
\footnote{The derivations of $\msuplossast$ and $\ncelossast$ are detailed in \cref{appendix:essential_bounds_of_losses}.}
\begin{align}
  & \msuplossast \eqdef \ln\{1 + (C - 1)e^{-2L^2}\}, \\
  & \ncelossast \eqdef \sum_{m=0}^K r_{K,C,m} \ln\{1 + m + (K - m)e^{-2L^2}\},
\end{align}
respectively, where $r_{K,C,m} \eqdef \binom{K}{m} \left(\frac{1}{C}\right)^m \left(1 - \frac{1}{C}\right)^{K-m}$.
Hence, the feasible region is
\begin{subequations}
  \begin{align}
    \msuploss(\fbf) &\le \nceloss(\fbf) + \Delta_\mathrm{U},
    \label{eq:feasible_region_learning_upper_bound} \\
    \msuploss(\fbf) &\ge \nceloss(\fbf) + \Delta_\mathrm{L},
    \label{eq:feasible_region_learning_lower_bound} \\
    \msuploss(\fbf) &\ge \msuplossast,
    \label{eq:feasible_region_essential_bound_supv} \\
    \nceloss(\fbf) &\ge \ncelossast,
    \label{eq:feasible_region_essential_bound_cont}
  \end{align}
  \label{eq:feasible_region}
\end{subequations}
as illustrated in \cref{figure:feasible_region}.
The first two bounds~\eqref{eq:feasible_region_learning_upper_bound} and \eqref{eq:feasible_region_learning_lower_bound} restrict the mean supervised loss by the contrastive loss.
We specifically refer to these bounds as \emph{surrogate bounds}.
The remaining two bounds~\eqref{eq:feasible_region_essential_bound_supv} and \eqref{eq:feasible_region_essential_bound_cont} represent the achievable limits for each loss separately.
One of the important questions is how the smallest possible value of $\msuploss$ in the feasible region \eqref{eq:feasible_region} changes as $K$ and $C$ change.
In other words, we are interested in whether the optimal point $(\ncelossast, \msuplossast)$ (\mystar in \cref{figure:feasible_region}) is always achievable regardless of the values of $K$ and $C$.
To investigate it, we check whether the optimal point  (\mystar) crosses the surrogate gaps (\mybullet or \mysquare) under the following two conditions.

\begin{figure*}[t]
  \centering
  \begin{minipage}{0.45\textwidth}
    \centering
    \includegraphics[width=0.9\textwidth]{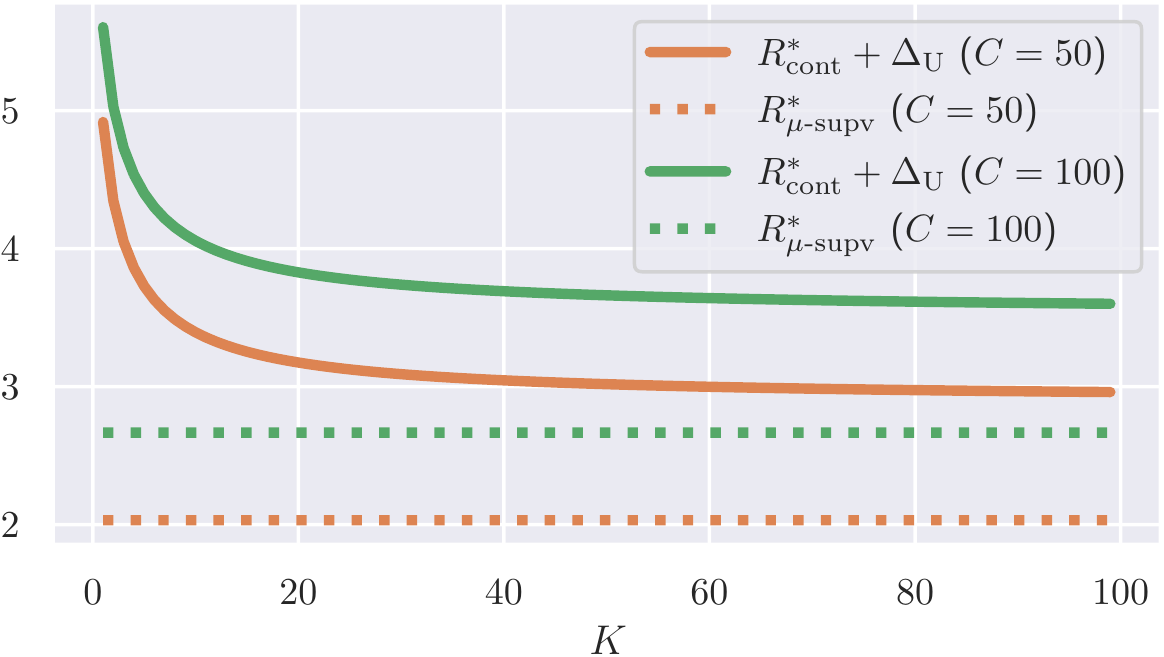}
    \subcaption{The feasible region at $\nceloss(\fbf) = \ncelossast$.}
    \label{figure:k_msuploss}
  \end{minipage}
  \begin{minipage}{0.45\textwidth}
    \centering
    \includegraphics[width=0.9\textwidth]{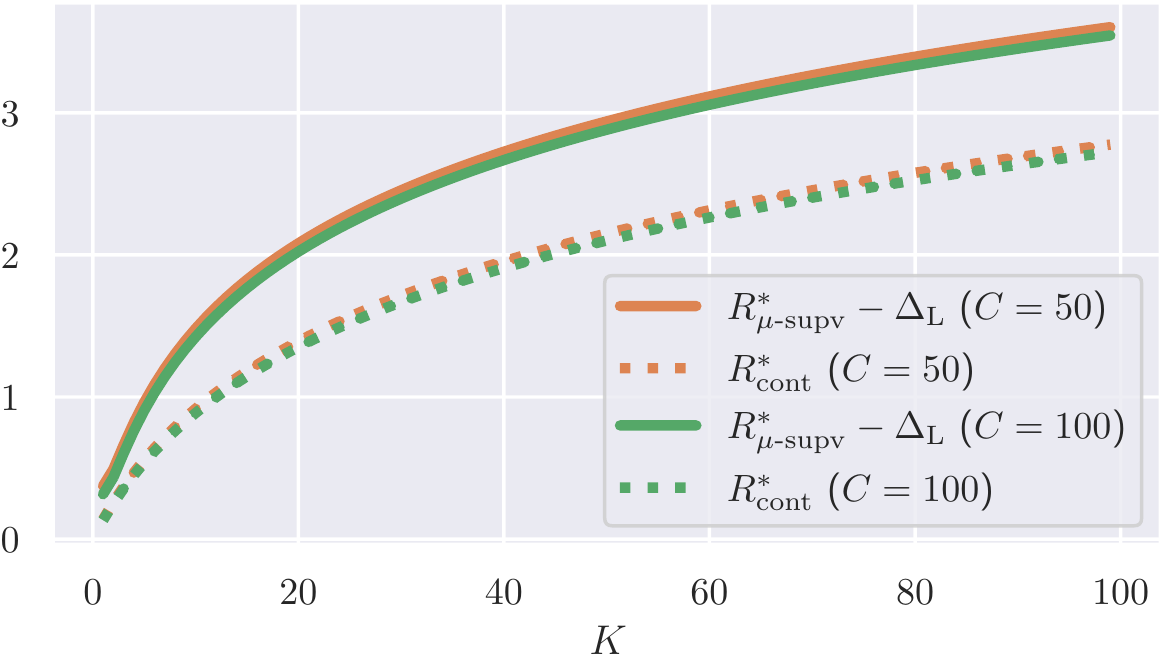}
    \subcaption{The feasible region at $\msuploss(\fbf) = \msuplossast$.}
    \label{figure:k_nceloss}
  \end{minipage}
  \caption{Visualization of the smallest possible value of $\msuploss$ in the feasible region \eqref{eq:feasible_region} for different $K$ and $C$.
  The dotted lines show the essential lower bounds which come from each loss separately.
  The solid lines show the surrogate upper bounds in which $\msuploss$ and $\nceloss$ restrict each other.}
\end{figure*}

\begin{itemize}
  \setlength{\leftskip}{-10pt}

  \item \underline{The feasible region at $\nceloss(\fbf) = \ncelossast$ (\cref{figure:k_msuploss}):}
  We plot the value $\ncelossast + \Delta_\mathrm{U}$ (solid line; the $\msuploss$-value of the point \mybullet in \cref{figure:feasible_region}) and the minimum possible $\msuploss$ ($\msuplossast$; dotted line) numerically.
  These two curves do not cross for all $K$, which means $\msuploss(\fbf) = \msuplossast$ is attainable no matter the values $K$ and $C$.
  In addition, the bound becomes sharper as $K$ increases, but the gap between the upper bound and $\msuplossast$ does remain even at the limit $K \nearrow \infty$.

  \item \underline{The feasible region at $\msuploss(\fbf) = \msuplossast$ (\cref{figure:k_nceloss}):}
  When $\msuploss(\fbf) = \msuplossast$, the contrastive loss $\nceloss(\fbf)$ is upper-bounded by $\msuplossast - \Delta_\mathrm{L}$ (the $\nceloss$-value of the point \mysquare in \cref{figure:feasible_region}).
  The curve of this value does not cross $\ncelossast$, which tells us that the lower bound does not exclude the optimal point $(\ncelossast, \msuplossast)$ from the feasible region \eqref{eq:feasible_region} at any $K$.
  Note that the gap between $\ncelossast$ and $\msuplossast - \Delta_\mathrm{L}$ gradually increases,
  meaning that it becomes much easier to attain $\msuplossast$ as $K$ increases.
\end{itemize}

Hence, the optimal point \mystar stays in the feasible region \eqref{eq:feasible_region} no matter the value $K$.
From this viewpoint, \emph{smaller $K$ is not necessarily disadvantageous because the optimal point \mystar remains in the feasible region.}
Note again that the estimation of $\msuploss$ may become harder with the smaller $K$ because of the gap $\Delta_\mathrm{U} - \Delta_\mathrm{L} = O(K^{-1})$, even if the optimal solution is unaffected by $K$.

\paragraph{Summary.}
We draw a connection between the mean supervised loss and the negative sample size $K$ by the following claim:
the gap between the contrastive loss and mean supervised loss shrinks with larger $K$ but the optimal mean supervised loss can nevertheless be achieved with small $K$.

\section{Comparison with Existing Work}
\label{section:comparison_with_existing}

\begin{table*}[t]
  \centering
  \caption{
    Surrogate bounds of the existing works.
    $H_n$ denotes the $n$-th harmonic number.
    Remark that \citeauthor{Arora2019ICML}'s and \citeauthor{Nozawa2021NeurIPS}'s bounds are valid only $K+1 \ge C$.
    The detailed derivations are discussed in \cref{appendix:existing_bounds}.
  }
  \label{table:existing_bounds}
  \begin{tabular}{lll}
    \toprule
    {} & \textsc{Upper Bound} & \textsc{Reference} \\
    \midrule
    $\msuploss(\fbf) \le$
    {} & $\frac{1}{(1 - \tau_K)v_{K+1}} \left\{ \nceloss(\fbf) - \E\ln(\col + 1) \right\}$ & \citet{Arora2019ICML} \\
    \addlinespace[5pt]
    {} & $\frac{1}{v_{K+1}} \left\{ 2\nceloss(\fbf) - \E\ln(\col + 1) \right\}$ & \citet{Nozawa2021NeurIPS} \\
    \addlinespace[5pt]
    {} & $\frac{2}{1 - \tau_K} \left\lceil\frac{2(C - 1)H_{C-1}}{K} \right\rceil \left\{ \nceloss(\fbf) - \E\ln(\col + 1) \right\}$ & \citet{Ash2022AISTATS} \\
    \bottomrule
  \end{tabular}
\end{table*}

This section first discusses the detailed difference between our main results and the existing theoretical results on CURL.
Then, we briefly review the other related literatures.

\paragraph{Surrogate bounds comparison.}
Here, we compare our results with the existing works by \citet{Arora2019ICML}, \citet{Nozawa2021NeurIPS}, and \citet{Ash2022AISTATS}.
We assume the uniform class prior for comparison.
We introduce a notation $\col \eqdef \sum_{k \in [K]} \indicator{c^+ = c^-_k}$.
Let $v_K$ be the probability that sampled $K$ negative classes contains all classes $c \in [C]$.
\begin{align}
  v_K \eqdef \sum_{n=1}^K \sum_{m=0}^{C-1} \binom{C-1}{m} (-1)^m \left(1 - \frac{m+1}{C}\right)^{n-1}\!\!.
\end{align}
The value $v_K$ is often referred to as the \emph{coupon collector's probability}.
Let $\tau_K$ be the probability that at least one of the negative classes $c^-_k$ is the same as the positive class $c^+$.
Under the uniform class prior, $\tau_K = 1 - (1 - \nicefrac{1}{C})^K$.
The surrogate bounds are summarized in \cref{table:existing_bounds}.%
\footnote{More precisely, \citet{Arora2019ICML} bound the averaged supervised loss over a part of the latent classes rather than $\msuploss$. Thus we can obtain a slightly better upper bound than \citeauthor{Arora2019ICML}'s bound shown in \cref{table:existing_bounds}. Nevertheless, the scale of the upper bound is dominated by the coefficient $(1-\tau_K)v_{K+1}$.}

\begin{figure}[t]
    \centering
    \begin{minipage}{0.4\textwidth}
      \includegraphics[width=\textwidth]{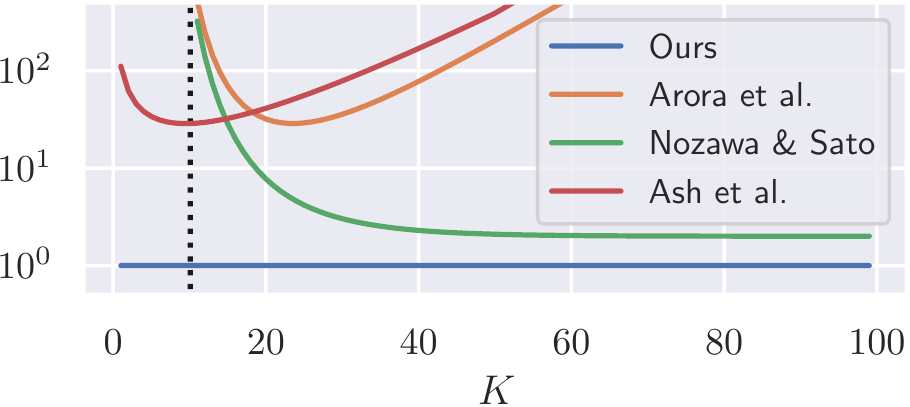}
      \subcaption{Coefficient of $\nceloss(\fbf)$.}
      \label{figure:upper_bound_coef}
    \end{minipage}
    \begin{minipage}{0.4\textwidth}
      \includegraphics[width=\textwidth]{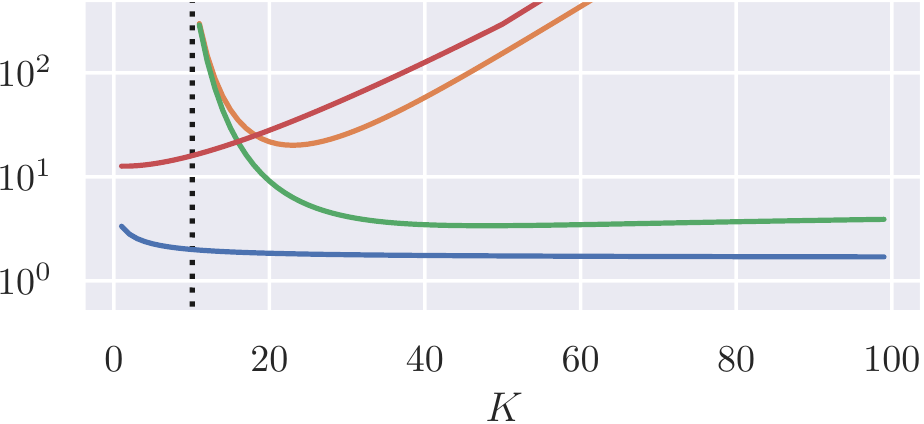}
      \subcaption{Upper bound at $\nceloss(\fbf) = \ncelossast$.}
      \label{figure:upper_bound_nceloss_min}
    \end{minipage}
    \caption{
      Theoretical comparison of surrogate bounds ($C = \num{10}$), log-scaled.
      \citeauthor{Arora2019ICML}'s and \citeauthor{Nozawa2021NeurIPS}'s bounds are valid only at $K + 1 \ge C$ (the dotted vertical lines).
    }
    \label{figure:upper_bound_comparison}
\end{figure}
We discuss the applicability and the dependence of existing and our bounds with respect to $K$.
As summarized in \cref{table:existing_bounds}, the coefficients of $\nceloss(\fbf)$ for existing bounds strongly depend on $C$ and $K$ through the \textit{coverage} ($v_K$) and the \textit{collision} ($\tau_K$) probability whereas our bounds have the constant coefficient.
As \citeauthor{Arora2019ICML}'s and \citeauthor{Nozawa2021NeurIPS}'s bounds depend on the \textit{coverage} probability $v_K$ in the denominator of the coefficients, the coefficients diverge in the range of $K+1<C$ (where the negative sample size is insufficient to cover the entire $[C]$).
In addition, the coefficients of the \citeauthor{Arora2019ICML}'s and \citeauthor{Ash2022AISTATS}'s bounds increase exponentially with increasing $K$ due to the \textit{collision} probability $\tau_K$ in the denominator, which are not consistent with the experimental facts.
Compared to these bounds, our upper bound has the coefficient independent of $C$ and $K$.

We numerically demonstrated the abovementioned dependencies on $K$ in \cref{figure:upper_bound_comparison}.
As we can see in \cref{figure:upper_bound_coef}, the coefficients of $\nceloss(\fbf)$ of \citeauthor{Arora2019ICML}'s and \citeauthor{Ash2022AISTATS}'s bounds have unique minima, \citeauthor{Nozawa2021NeurIPS}'s coefficient has monotonically decreasing nature, and our coefficient is constant.
On the other hand, the tendencies of the bound values at $\nceloss(\fbf) = \ncelossast$, namely, the best possible mean supervised loss in terms of the upper bounds (\cref{figure:upper_bound_nceloss_min}) are slightly different from the coefficient: \citeauthor{Ash2022AISTATS}'s bound is monotonically increasing, \citeauthor{Arora2019ICML}'s and \citeauthor{Nozawa2021NeurIPS}'s bounds have a unique minimum, and ours is monotonically decreasing.
Among the compared bounds, only ours is legitimate for all $K$ and moderately decreases with $K$, which agrees well with the experimental fact observed as well in \cref{figure:upper_bound_comparison_cifar10}; the details are stated in \cref{section:large_experiments}.%
\footnote{Note that \citet{Nozawa2021NeurIPS}'s bound also implies larger $K$ is better.
Still, our argument on how contrastive learning works differs from theirs.
See \cref{appendix:existing_bounds} for the further discussions.}
Such a moderate dependence on $K$ is due to the mechanism that the contrastive loss behaves as a surrogate objective.

\paragraph{Related literatures.}
\citet{Wang2020ICML} showed that the contrastive loss asymptotically favors data representation uniformly distributed over the unit sphere yet aligning across semantically similar samples.
\citet{Li2021NeurIPS} proposed an alternative loss function to the contrastive loss based on a kernel metric, following the similar idea to \citet{Wang2020ICML}.
\citet{Tosh2021ALT} showed that a (linear) mean classifier learned in CURL can approximate the (potentially nonlinear) Bayes classifier well.

While our work does not handle DA, several works analyzed the effect of DA on the performance.
\citet{Wen2021ICML} showed that DA is necessary to recover sparse signals under a specific assumption on the model architecture.
\citet{HaoChen2021NeurIPS} introduced a notion of the augmentation graph, representing how likely the nearby samples are generated via DA and showed that a type of contrastive loss could be viewed as a low-rank approximation of the adjacency matrix of the augmentation graph.
\citet{Kugelgen2021NeurIPS} proposed a loss function that enables the model to identify invariant factors across DA.

We mention a few works analyzing the other types of self-supervised learning; \citet{Garg2020NeurIPS} analyzed masked self-supervised learning,
\citet{Wei2021ICLR} analyzed the input consistency loss for unsupervised learning,
and \citet{Saunshi2021ICLR} analyzed auto-regressive language models.
\citet{Grill2020NeurIPS,Chen2021CVPR} proposed self-supervised learning without negative samples.

Lastly, multi-sample estimators~\citep{Oord2018arXiv,Poole2019ICML,Song2020ICLR} popularly used in mutual information estimation are substantially related to the contrastive loss.
We defer its discussion to \cref{appendix:mutual_information_estimation}.

\paragraph{Remark.}
A concurrent work~\citep{Wang2022ICLR} recently established the surrogate bound that has similar order in $K$ with ours without conditional independence assumption.
We stress that our results were obtained independently of theirs.
In addition, the purpose of our research is to clarify the mechanism of how $K$ affects the downstream performance, which is different from their motivation to discuss the validity of assumptions in contrastive learning.
In \cref{appendix:relaxation}, we discuss how our surrogate bounds hold without the conditional independence assumption.

\section{Experiments}
\label{section:experiments}

\begin{figure*}[htp]
  \begin{minipage}{\textwidth}
    \centering
    \includegraphics[width=\textwidth]{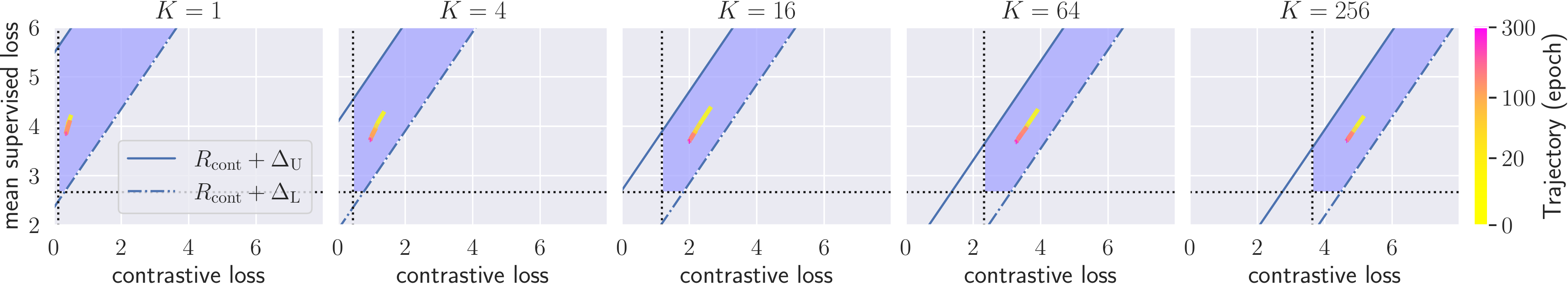}
    \caption{
      Learning trajectories of the \texttt{circle} dataset in the $(\nceloss,\msuploss)$-plot.
      The trajectories are plotted with gradient color lines, indicating the epochs.
    }
    \label{figure:trajectory}
  \end{minipage}
  \begin{minipage}{\textwidth}
    \centering
    \includegraphics[width=\textwidth]{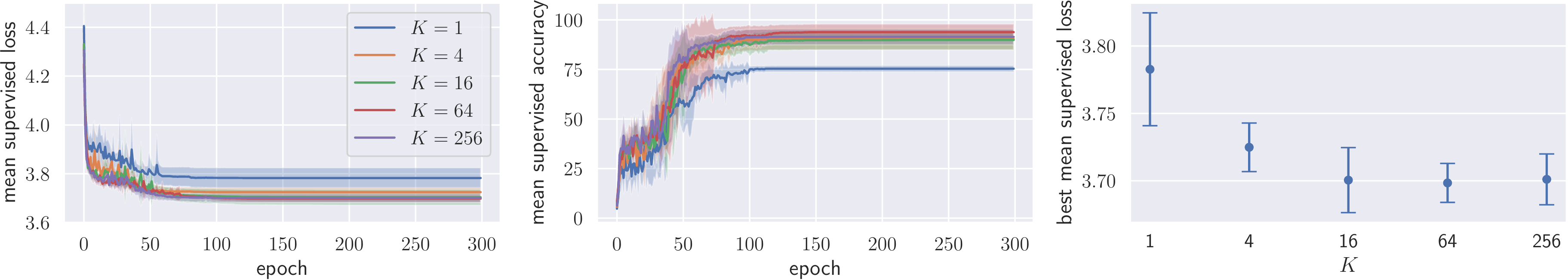}
    \caption{
      For each $K$, eight runs on the \texttt{circle} dataset are averaged with the standard deviations plotted.
      \textbf{(Left)} the test mean supervised losses at each epoch with the different negative sample sizes $K$.
      \textbf{(Middle)} the test mean supervised accuracy at each epoch with the different negative sample sizes $K$.
      \textbf{(Right)} the best test mean supervised loss with the different negative sample sizes $K$.
    }
    \label{figure:loss_and_curve}
  \end{minipage}
\end{figure*}

\begin{figure*}[t]
  \centering
  \begin{subfigure}[b]{0.30\textwidth}
    \includegraphics[width=\textwidth]{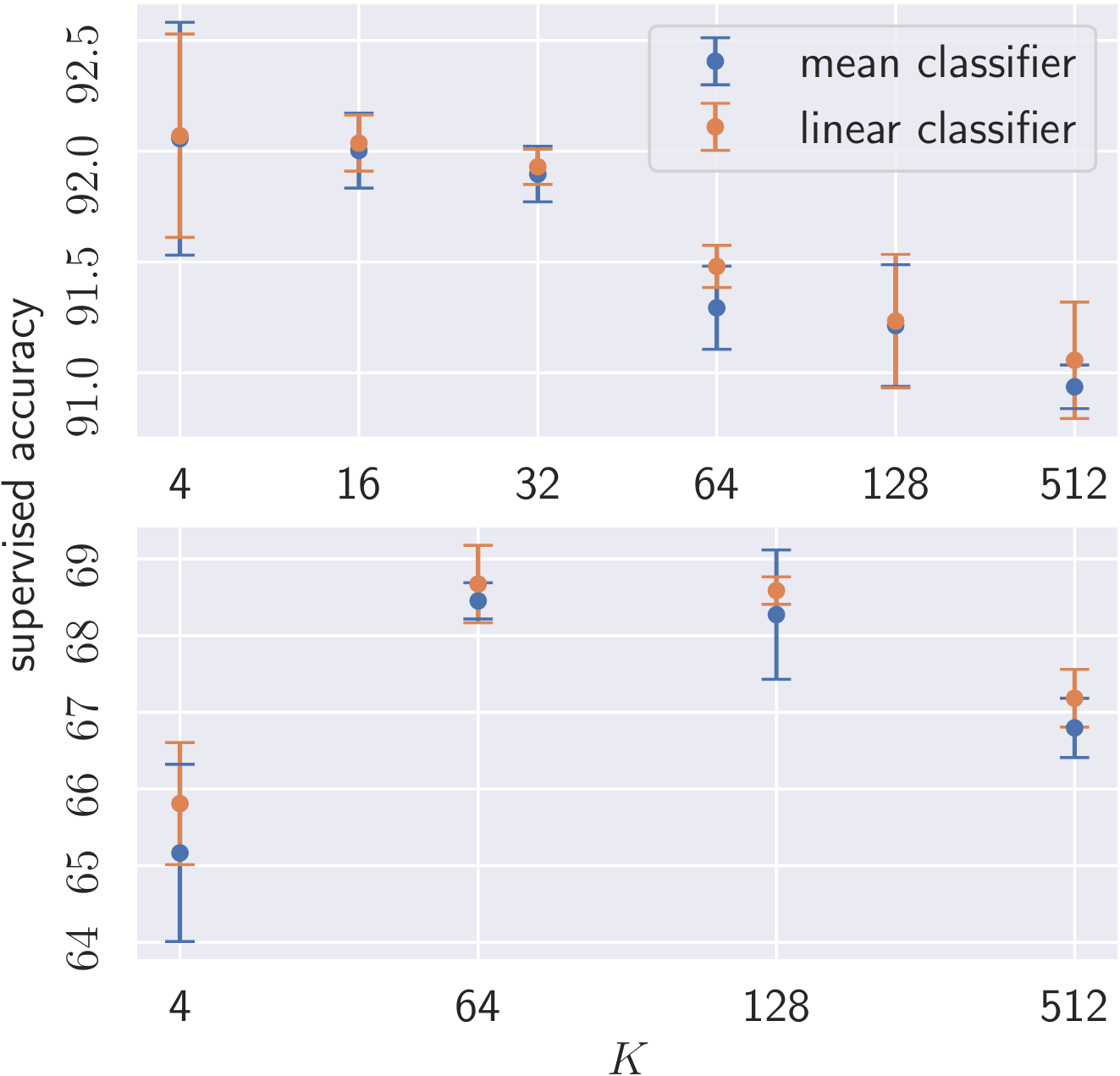}
    \caption{CIFAR-10 (\textbf{Top})/100 (\textbf{Bottom}).}
    \label{figure:cifar}
  \end{subfigure}
  \begin{subfigure}[b]{0.60\textwidth}
    \resizebox{\linewidth}{!}{%
      \includegraphics[width=\textwidth]{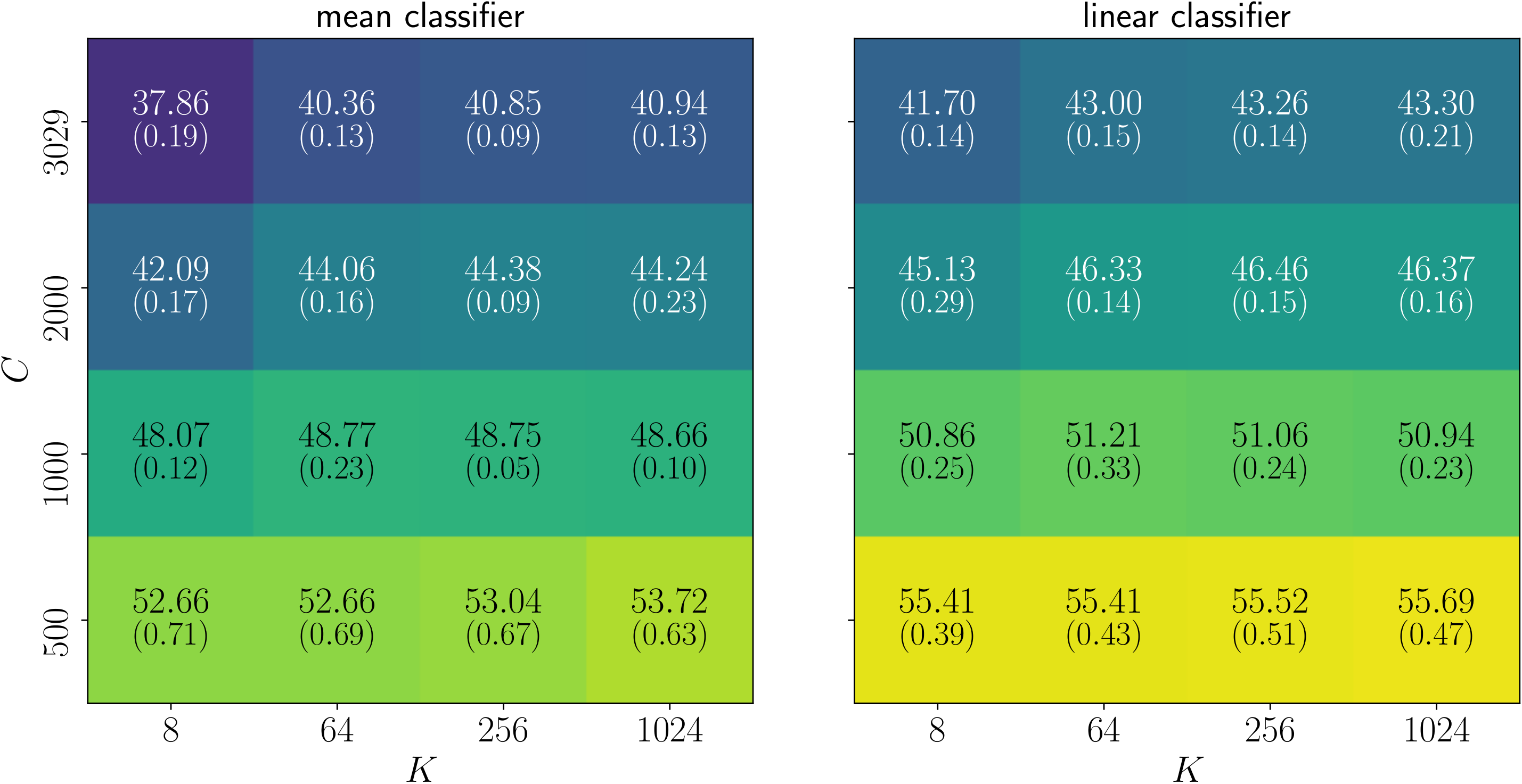}
    }
    \caption{Wiki-3029.}
    \label{figure:wiki-heatmap}
  \end{subfigure}
  \caption{
    Mean and linear classifier's test accuracy on CIFAR-10/100 and Wiki-3029 when varying the negative samples size $K$.
    For Wiki-3029, we also change the number of latent classes $C$.
    The error bars in (\textbf{a}) and parenthesized number in (\textbf{b}) indicate the standard deviation of three runs.}
  \label{figure:cifar-and-wiki}
\end{figure*}

\begin{figure*}[t]
  \centering
  \begin{subfigure}[b]{0.40\textwidth}
    \includegraphics[width=\textwidth]{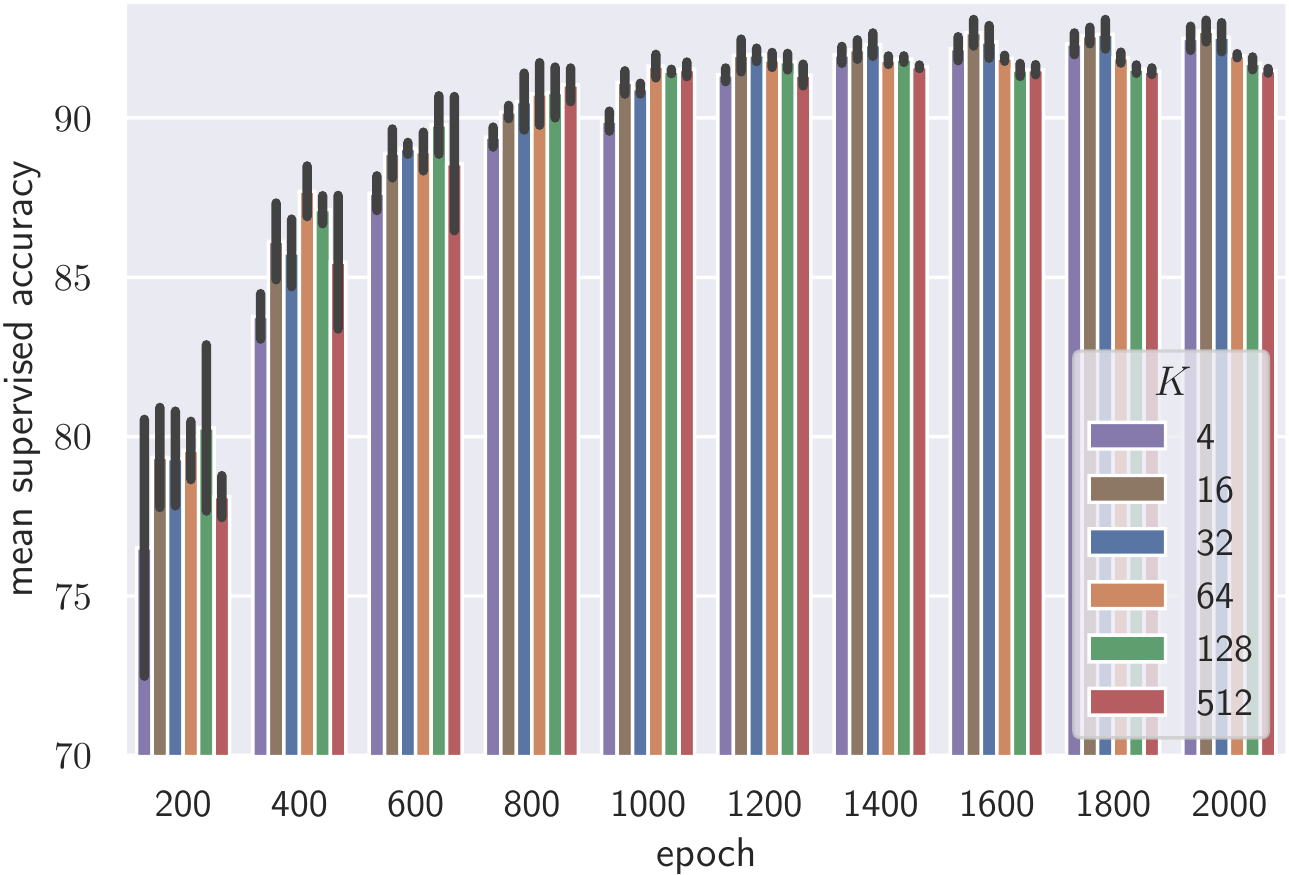}
    \caption{CIFAR-10.}
    \label{figure:cifar10-epoch}
  \end{subfigure}
  \begin{subfigure}[b]{0.40\textwidth}
    \includegraphics[width=\textwidth]{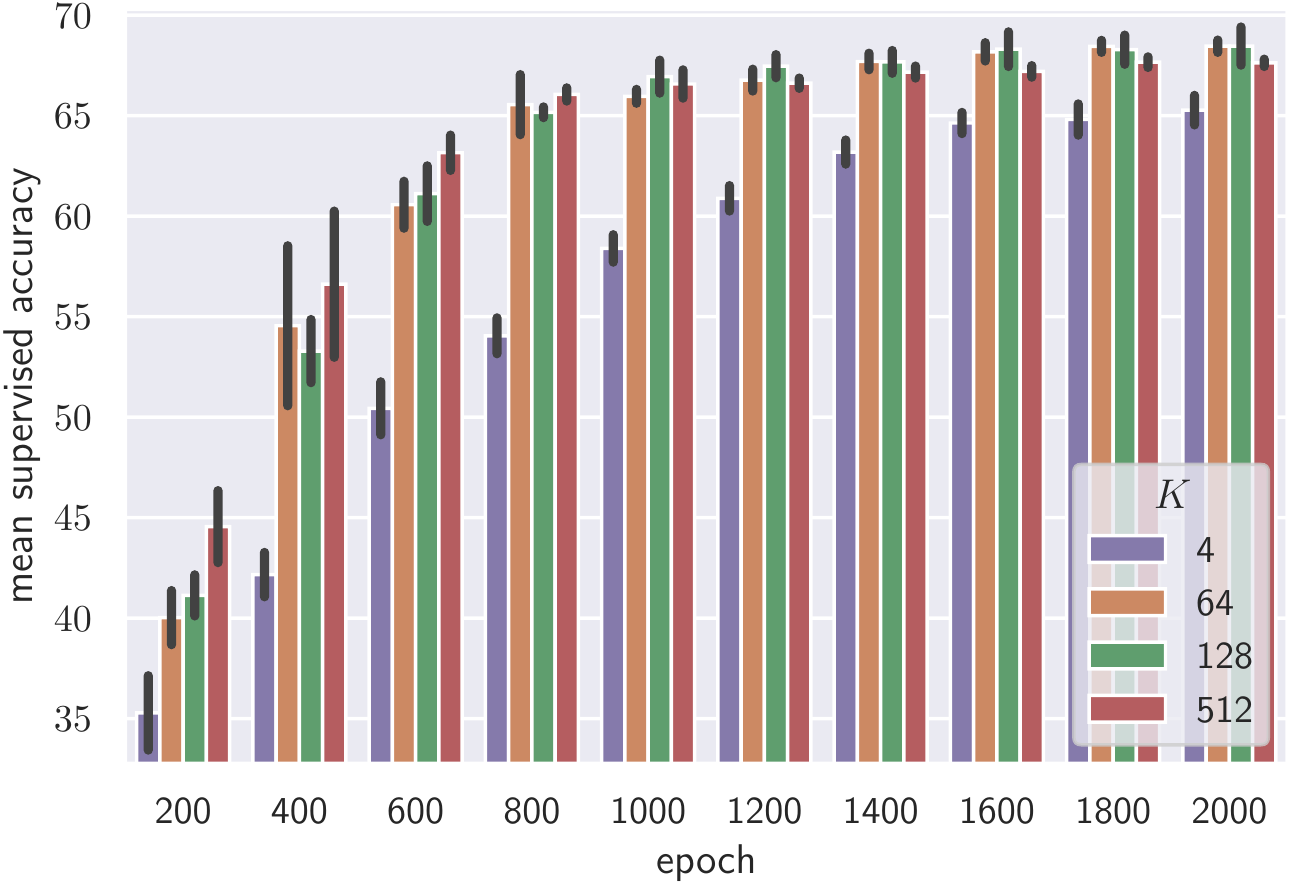}
    \caption{CIFAR-100.}
    \label{figure:cifar100-epoch}
  \end{subfigure}
  \caption{Test accuracy of mean classifier at every $\num{200}$ epochs on CIFAR-10/100.
  In CIFAR-100, the accuracy of $K=\num{4}$ and the others' have a large gap at smaller epochs at epoch $\num{200}$, but the gap become smaller when epochs increase.
  The error bars indicate the standard deviation of three runs.
  }
  \label{figure:cifar-epoch}
\end{figure*}

We verified our theoretical findings with experiments on synthetic (\cref{section:small_experiments}), vision, and language datasets (\cref{section:large_experiments}).
The details of the setup are in \cref{appendix:experimental_details}.
The experimental codes to reproduce all figures in the paper are available at \url{https://github.com/nzw0301/gap-contrastive-and-supervised-losses}.

\subsection{Small-scale Experiments on Synthetic Dataset}
\label{section:small_experiments}

\paragraph{Dataset and learning setups.}
We create a synthetic dataset \texttt{circle}, which is a 2D dataset created as follows:
for each class $c \in [C]$ ($C = 10$), $\num{1000}$ samples are drawn from $\mathsf{Uniform}([-\num{0.5} \; -\num{0.5}], [\num{0.5}, \num{0.5}])$, normalized, and multiplied by $\nicefrac{c + 1}{2}$.
The generated samples are nonlinear and require disentanglement to be linearly separable.
We treated $\num{60}\%$ of the generated samples as a training dataset and the rest of the samples as a test dataset.

As a feature extractor $\fbf$, we used a multi-layer perceptron (the number of units $\num{2}$-$\num{256}$-$\num{256}$-$\num{256}$) with the ReLU activation functions following after each hidden layer.
During the training, the extracted feature representations are normalized.
For negative samples, we sampled $K \in \{\num{1}, \num{4}, \num{16}, \num{64}, \num{256}\}$ samples without replacement from $2B-2$ points included in the same mini-batch to avoid the influence of mini-batch size $B$, inspired by~\citet{Ash2022AISTATS}.%
\footnote{
  Each mini-batch consists of $B$ pairs of positive pairs.
  The candidates of the negative samples are the $2B - 2$ samples excluding the anchor and its paired point.
}

\paragraph{Results.}
\cref{figure:trajectory} shows a single trajectory in the $(\nceloss,\msuploss)$-plot and the feasible region (confer \cref{figure:feasible_region}) for each $K$.
We plotted the trajectories by tracking $(\nceloss(\fbf^{(t)}), \msuploss(\fbf^{(t)}))$ at each epoch $t$ computed with the test dataset.
All trajectories were located in between the upper ($\nceloss + \Delta_\mathrm{U}$) and lower ($\nceloss + \Delta_\mathrm{L}$) bounds as a matter of course.
Given that the existing surrogate bounds provide the much larger upper bounds (\cref{figure:upper_bound_comparison}), our surrogate bounds provide the finest estimate of the mean supervised loss.
In addition, it is remarkable that all trajectories have nearly the same slopes as our surrogate bounds, which constitutes solid evidence that our surrogate bounds capture the learning dynamics well.

In \cref{figure:loss_and_curve}, the mean supervised loss and accuracy are compared with the different $K$.
We plotted the standard deviations of the same experiments with eight different random seeds for each $K$.
From these figures, it can be concluded that the contrastive loss performance becomes better with the larger $K$ in the sense that the supervised loss improved and the variance shrank.
The variance improvement is theoretically suggested by \cref{figure:trajectory} as well; the larger $K$ is, the smaller the gap between upper and lower bounds becomes.

\subsection{Large-scale Experiments on Vision and Language Datasets}
\label{section:large_experiments}

We used the same datasets as \citet{Arora2019ICML}: CIFAR-100~\citep{Krizhevsky2009techrep} and Wiki-3029~\citep{Arora2019ICML} datasets, along with CIFAR-10~\citep{Krizhevsky2009techrep} dataset.

\paragraph{Learning setups.}
We treated the supervised classes as latent classes as in \citet{Arora2019ICML} and \citet{Ash2022AISTATS} for creating positive pairs.
We used the original supervised classes of CIFAR-10/100 as $[C]$; $C=\num{10}$ and $C=\num{100}$, respectively.
We used $K$ in $\{ \num{4}, \num{16}, \num{32}, \num{64}, \num{128}, \num{512} \}$ and in $\{ \num{4}, \num{64}, \num{128}, \num{512} \}$ for CIFAR-10/100, respectively.
For Wiki-3029, we used $C \in \{\num{500}, \num{1000}, \num{2000}, \num{3029}\}$ and $K \in \{\num{8}, \num{64}, \num{256}, \num{1024}\}$.
For each different $(C, K)$, we trained the feature extractor $\fbf$ on the training dataset.
We then evaluated its performance on the test dataset with mean and linear classifiers.
We used ResNet-18~\citep{He2016CVPR}-based feature extractor $\fbf$ for CIFAR-10/100 and the fasttext~\citep{Joulin2017EACL}-based feature extractor for Wiki-3029.

\paragraph{Results.}
\Cref{figure:upper_bound_comparison_cifar10} shows the comparison between the estimated upper bounds using \cref{theorem:curl_upper_bound} and actual supervised loss on the CIFAR-10 test dataset.
We estimated the bounds by substituting the actual $\nceloss(\fbf)$ to the equations shown in \cref{theorem:curl_upper_bound} and \cref{table:existing_bounds}.
Our bound gave the closest bound to the experimental value of the supervised loss.
The existing surrogate bounds of \citet{Arora2019ICML} and \citet{Ash2022AISTATS} were prohibitively large to explain the classification performance.
Although \citeauthor{Nozawa2021NeurIPS}'s bound was comparable with ours, it was valid only in $K + 1 \ge C$ and tended to diverge near $K + 1 = C$, as shown in \cref{section:comparison_with_existing}.

We investigated how $K$ affects the test accuracy for different $C$ in \cref{figure:cifar-and-wiki}.
The test accuracy improved or was saturated with the larger $K$ for all $C$ on Wiki-3029.
In contrast, it was degraded as $K$ increased in mean and linear classifiers on CIFAR-10/100.
This behavior could be partly because of the gap between the cross-entropy loss and the supervised accuracy---the theory of CURL, including the existing studies, usually focuses on the cross-entropy loss only.
\cref{figure:cifar-closer-look} in \cref{appendix:experimental_details_figure4} revealed that the supervised loss was not significantly worse with the larger $K$ on CIFAR-10/100.

With the smaller $K$ and large $C$, we found that long epochs were more effective to improve classification accuracy than increasing the negative sample size $K$ (\cref{figure:cifar100-epoch}).
While similar results were reported by \citet[Figure 9]{Chen2020ICML},
it is important to remark that we randomly drew $K$ negative samples from the $2B - 2$ samples in the given mini-batch at each iteration as in \citet{Ash2022AISTATS}---a different approach was used by \citet{Chen2020ICML} to regard the all samples in the mini-batch except an anchor sample as negative samples.
Under our experimental setup, a learner may encounter less diverse samples with the smaller $K$ even if the mini-batch size $B$ is the same,
which could make the downstream performance worse---the longer epochs are necessary to mitigate the issue.
Since the CIFAR-10 dataset has a smaller $C$ and is simpler than the CIFAR-100, all accuracies were saturated with similar epochs for all $K$ (\cref{figure:cifar10-epoch}).

\section{Conclusion}
We established novel surrogate bounds for contrastive learning.
In contrast to existing theories, our bounds are applicable for all negative sample sizes and have a constant coefficient.
We verified that our bounds well explained learning dynamics on the synthetic dataset, and the surrogate gap shrinks with large negative samples.
For the vision and language datasets, the downstream classification losses were also best explained by our bounds in contrast to existing theories.
Our bounds provided a perspective on the effect of negative sample size that the contrastive loss behaves as a surrogate objective of the downstream loss, and its surrogate gap decays with larger negative samples.

\subsubsection*{Acknowledgments}
HB was supported by JSPS KAKENHI Grant Number 19J21094.
KN was supported by JSPS KAKENHI Grant Number 18J20470.
A part of experiments of this research was conducted using Wisteria/Aquarius in the Information Technology Center, The University of Tokyo.
We appreciate anonymous reviewers of ICLR 2022 and
ICML 2022 for giving constructive suggestions to improve our manuscript.

\bibliography{reference}
\bibliographystyle{icml2022}

\newpage
\appendix
\onecolumn

\begin{center}
  \fontsize{14.4pt}{20pt}
  \selectfont
  \rule[5pt]{\columnwidth}{1pt} \\
  \textsc{
    <<Appendix>> \\
    On the Surrogate Gap between Contrastive and Supervised Losses
  } \\
  \rule[0pt]{\columnwidth}{1pt}
\end{center}

As an additional notation, the $d$-dimensional ball of radius $r$ associated with the $\Lbb_p$-norm is denoted by $\Bbb_p^d(r) \eqdef \{ \xbf \in \Rbb^d \mid \|\xbf\|_p \le r \}$.
For $\zbf \in \Rbb^C$, the log-sum-exp function is denoted by $\LSE{\zbf} \eqdef \ln ( \sum_{c \in [C]} \exp(z_c) )$.

\section{Useful Lemmas}
\label{appendix:lemmas}

In this section, a few lemmas are introduced in order to prove the main results.

\begin{lemma}
  \label{lemma:log_sum_exp_bound}
  For $\zbf \in [-L^2, L^2]^N$,
  \begin{align}
    2\ln N \le \LSE{\zbf} + \LSE{-\zbf} \le 2\ln(N\cosh(L^2)).
  \end{align}
\end{lemma}

\begin{proof}
  Define $H(\zbf) \eqdef \LSE{\zbf} + \LSE{-\zbf}$.
  First, we prove the lower bound of $H(\zbf)$.
  Since
  \begin{align}
    \pdiff{H}{z_i} = \frac{\exp(z_i)}{\sum_{n \in [N]} \exp(z_n)} - \frac{\exp(-z_i)}{\sum_{n \in [N]} \exp(-z_n)}
  \end{align}
  for all $i \in [N]$, $\zbf = \zerobf_N$ satisfies the first-order optimality condition of $H$.
  By noting that $H$ is convex due to the convexity of the log-sum-exp functions,
  $H$ is minimized at $\zbf = \zerobf_N$: $H(\zbf) \ge H(\zerobf_N) = 2\ln N$.
  Note that $\zbf$ can be any vector in $\Rbb^N$ for this lower bound.

  Next, we prove the upper bound of $H(\zbf)$.
  Observe that finding the maximum of $H(\zbf)$ in $\zbf \in \Bbb_\infty^N(L^2)$ is equivalent to a concave minimization problem over a convex polytope.
  It is known that every vertex of the polytope $\zbf \in \{-L^2, L^2\}^N$ is a local optimum for concave minimization over a convex polytope~\citep{Pardalos1986SIREV}.
  Hence, it is sufficient to test the vertices $\zbf \in \{-L^2, L^2\}^N$ to find the maximum of $H(\zbf)$.
  Define
  \begin{align}
    \tilde H(m) &\eqdef H\Bigl(\underbrace{L^2, L^2, \dots, L^2}_{m}, \underbrace{-L^2, -L^2, \dots, -L^2}_{N-m}\Bigr) \\
    &= \ln \Bigl\{ m\exp(L^2) + (N-m)\exp(-L^2) \Bigr\} + \ln \Bigl\{ m\exp(-L^2) + (N-m)\exp(L^2) \Bigr\}.
  \end{align}
  Note that the maximizer of $H(\zbf)$ in $\zbf \in \{-L^2, L^2\}^N$ is equivalent to that of $\tilde H(m)$ in $m \in [N]$
  because $H(\zbf)$ is symmetric in every $z_n$ for $n \in [N]$.
  We verify the following by simple algebra:
  \begin{align}
    \exp \tilde H(m)
    &= -\left(\exp(L^2) + \exp(-L^2)\right) \left\{ \left(m - \frac{N}{2}\right)^2 + \const \right\},
  \end{align}
  meaning that $\tilde H(m)$ is maximized at $m = \left\lfloor\frac{N}{2}\right\rfloor$.
  Hence, $H(\zbf) \le \tilde H(\lfloor\nicefrac{N}{2}\rfloor) \le \tilde H(\nicefrac{N}{2}) = 2\ln(N\cosh(L^2))$.
\end{proof}

\begin{lemma}
  \label{lemma:cross_entropy_offset}
  For all $z_0 \in \Rbb$ and $\zbf \in \Rbb^{K}$ such that $z_0, z_k \in [-L^2, L^2]$ ($\forall k \in [K]$),
  \begin{align}
    \ln & \frac{ \exp(-z_0) }{\exp(-z_0) + \sum_{k \in [K]} \exp(-z_k) } \nonumber \\
    & \ge -\ln\frac{\exp(z_0)}{ \exp(z_0) + \sum_{k \in [K]} \exp(z_k)} - 2\ln\left\{ (K+1) \cosh(L^2) \right\}.
  \end{align}
\end{lemma}

\begin{proof}
  We write $\tilde\zbf \eqdef [z_0 \; \zbf^\top]^\top$.
  Let $H(\tilde\zbf)$ be a function such that
  \begin{align}
    H(\tilde\zbf)
    &\eqdef -\ln\frac{\exp(z_0)}{\exp(z_0) + \sum_{k \in [K]} \exp(z_k) } - \ln\frac{ \exp(-z_0) }{ \exp(-z_0) + \sum_{k \in [K]} \exp(-z_k) } \\
    &= \ln\sum_{k=1}^{K+1} \exp(\tilde z_k)  + \ln\sum_{k=1}^{K+1} \exp(-\tilde z_k) .
  \end{align}
  Our goal is to find a tight upper bound of $H(\tilde\zbf)$ for $\tilde\zbf \in \Bbb_\infty^{K+1}(L^2)$.

  Observe that $H(\tilde\zbf)$ is the sum of the two log-sum-exp functions hence it is convex in $\tilde\zbf$.
  In addition, the domain $\Bbb_\infty^{K+1}(L^2)$ is a compact convex polytope.
  Henceforth, every vertex of the polytope, $\tilde\zbf \in \{-L^2, L^2\}^{K+1}$, is a local maximizer
  because maximizing $H(\tilde\zbf)$ is concave minimization over a convex polytope~\citep{Pardalos1986SIREV}.
  Since $H(\tilde\zbf)$ is symmetric in every element $\tilde z_k$,
  it is sufficient to test the vertices and see the difference between
  \begin{align}
    \underbrace{H\left( (\underbrace{L^2, \dots, L^2}_{\text{\# = $j$}}, -L^2, \dots, -L^2) \right)}_{\eqdef \tilde H(j)}
    \text{ and }
    \underbrace{H\left( (\underbrace{L^2, \dots, L^2}_{\text{\# = $j+1$}}, -L^2, \dots, -L^2) \right)}_{\eqdef \tilde H(j+1)}
  \end{align}
  for $j \in \{0, \dots, K\}$ to seek out the global maximum.
  For $0 \le j \le K$, a simple algebra shows
  \begin{align}
    \exp\left( \tilde H(j) \right) - \exp\left( \tilde H(j+1) \right)
    &= (K - 2j) \underbrace{\left\{ 2 - \left( \exp(2L^2) + \exp(-2L^2) \right) \right\}}_{\substack{\leq 0 \\ \text{because of AM-GM inequality}}},
  \end{align}
  from which we can tell that $\tilde H(j)$ is maximized at
  $j = K/2$ when $K$ is even and $j = (K+1)/2$ when $K$ is odd.
  In addition, it is confirmed that
  \begin{align}
    \exp\left( \tilde H\left(\frac{K}{2}\right) \right) - \exp\left( \tilde H\left(\frac{K+1}{2}\right) \right)
    &= \frac{2 - \left( \exp(2L^2) + \exp(-2L^2) \right)}{4} \\
    &\le 0,
  \end{align}
  where the AM-GM inequality is invoked at the last line.
  Eventually, $\tilde H\left((K+1)/2\right)$ turns out to be a tight upper bound of $H(\tilde\zbf)$ for $\tilde\zbf \in \Bbb_\infty^{K+1}(L^2)$.
  It is elementary to confirm $\tilde H\left((K+1)/2\right) = 2\ln\left\{ (K+1) \cosh(L^2) \right\}$.
\end{proof}

\section{Proofs of Main Results}
\label{appendix:proofs}

In this section, we provide proofs for the main results, \cref{theorem:curl_upper_bound,theorem:curl_lower_bound}.

\thmcurlupperbound*

\begin{proof}[Proof of \cref{theorem:curl_upper_bound}]
  The proof largely relies on the Jensen's inequality.
  To apply the Jensen's inequality in the \emph{reversed} way, we occasionally transform a convex function into a concave function by applying the log-sum-exp bound in \cref{lemma:log_sum_exp_bound}.
  \begin{align}
    \nceloss(\fbf)
    &= - \E_{c^+, \{c^-_k\}, \xbf, \xbf^+, \{\xbf_k^-\}} \ln\frac{\exp(\fbf(\xbf)^\top\fbf(\xbf^+))}{\exp(\fbf(\xbf)^\top\fbf(\xbf^+)) + \sum_{k \in [K]} \exp(\fbf(\xbf)^\top\fbf(\xbf^-_k))} \\
    &= - \E_{c^+, \xbf, \xbf^+} [\fbf(\xbf)^\top\fbf(\xbf^+)] + \E_{c^+, \{c^-_k\}, \xbf, \xbf^+, \{\xbf^-_k\}} \ln \Bigl(\exp(\fbf(\xbf)^\top\fbf(\xbf^+)) + \sum_{k \in [K]} \exp(\fbf(\xbf)^\top\fbf(\xbf^-_k))\Bigr) \\
    & \ge - \E_{c^+, \xbf, \xbf^+} [\fbf(\xbf)^\top\fbf(\xbf^+)] + \E_{c^+, \{c^-_k\}, \xbf, \xbf^+, \{\xbf^-_k\}} \underbrace{\ln \sum_{k \in [K]} \exp(\fbf(\xbf)^\top\fbf(\xbf^-_k))}_{= \LSE{\{\fbf(\xbf)^\top\fbf(\xbf^-_k)\}_{k \in [K]}}} \qquad \text{(monotonicity of $\ln$)} \\
    & \ge - \E_{c^+, \xbf, \xbf^+} [\fbf(\xbf)^\top\fbf(\xbf^+)] + \E_{c^+, \{c^-_k\}, \xbf} \underbrace{\ln \sum_{k \in [K]} \exp(\fbf(\xbf)^\top\mubf_{c^-_k})}_{= \LSE{\{\fbf(\xbf)^\top\mubf_{c^-_k}\}_{k \in [K]}}} \qquad \text{(Jensen's inequality)} \\
    &\ge - \E_{c^+, \xbf, \xbf^+} [\fbf(\xbf)^\top\fbf(\xbf^+)] - \E_{c^+, \{c^-_k\}, \xbf} \underbrace{\ln \sum_{k \in [K]} \exp(-\fbf(\xbf)^\top\mubf_{c^-_k})}_{= \LSE{\{-\fbf(\xbf)^\top\mubf_{c^-_k}\}_{k \in [K]}}} + 2\ln K \qquad \text{(\cref{lemma:log_sum_exp_bound})} \\
    &\ge - \E_{c^+, \xbf, \xbf^+} [\fbf(\xbf)^\top\fbf(\xbf^+)] - \E_{c^+, \xbf} \ln \sum_{k \in [K]} \E_{\{c^-_k\}} [\exp(-\fbf(\xbf)^\top\mubf_{c^-_k})] + 2\ln K \qquad \text{(Jensen's inequality)} \\
    &\stackrel{\text{(a)}}{\ge} - \E_{c^+, \xbf, \xbf^+} [\fbf(\xbf)^\top\fbf(\xbf^+)] - \E_{c^+, \xbf} \ln \Bigl( K \sum_{c \in [C]}\pi_{(1)} \exp(-\fbf(\xbf)^\top\mubf_{c}) \Bigr)  + 2\ln K \\
    &= - \E_{c^+, \xbf, \xbf^+} [\fbf(\xbf)^\top\fbf(\xbf^+)] - \E_{c^+, \xbf} \underbrace{\ln \sum_{c \in [C]} \exp(-\fbf(\xbf)^\top\mubf_{c})}_{= \LSE{\{-\fbf(\xbf)^\top\mubf_c\}_{c \in [C]}}} + 2\ln K - \ln K\pi_{(1)} \\
    &\ge - \E_{c^+, \xbf, \xbf^+} [\fbf(\xbf)^\top\fbf(\xbf^+)] + \E_{c^+, \xbf} \underbrace{\ln \sum_{c \in [C]} \exp(\fbf(\xbf)^\top\mubf_{c})}_{= \LSE{\{\fbf(\xbf)^\top\mubf_c\}_{c \in [C]}}} + 2\ln(C\cosh(L^2)) - 2\ln K - \ln K\pi_{(1)} \quad \text{(\cref{lemma:log_sum_exp_bound})} \\
    &\stackrel{\text{(*)}}{=} - \E_{c^+, \xbf} [\fbf(\xbf)^\top\mubf_{c^+}] + \E_{c^+, \xbf} \underbrace{\ln \sum_{c \in [C]} \exp(\fbf(\xbf)^\top\mubf_{c})}_{= \LSE{\{\fbf(\xbf)^\top\mubf_c\}_{c \in [C]}}} - 2\ln(C\cosh(L^2)) + 2\ln K - \ln K\pi_{(1)} \\
    &= \msuploss(\fbf) - 2\ln(C\cosh(L^2)) + 2\ln K - \ln K\pi_{(1)},
  \end{align}
  where we use $\E_c[A] \le \sum_{c \in [C]} \pi_{(1)} A$ and the monotonicity of $-\ln$ at (a).
  Note that the conditional independence is used only at (*).
\end{proof}

\thmcurllowerbound*

\begin{proof}[Proof of \cref{theorem:curl_lower_bound}]
  The proof is essentially a consequence of the Fenchel's inequality and the Jensen's inequality.
  First, by noting that the convex conjugate of the log-sum-exp function is the negative Shannon entropy,
  the following identity is obtained.
  \begin{align}
    \msuploss(\fbf)
    &= \E_{\xbf, y} \left[ -\fbf(\xbf)^\top\mubf_y + \LSE{\Wbf^\mu\fbf(\xbf)} \right] \\
    &= \E_{\xbf, y} \left[ -\fbf(\xbf)^\top\mubf_y + \sup_{\pbf \in \triangle^C} \left\{ \pbf^\top(\Wbf^\mu\fbf(\xbf)) + \entropy{\pbf} \right\} \right].
  \end{align}
  If we choose an arbitrary $\pbf \in \triangle^C$, $\msuploss(\fbf)$ is lower bounded (Fenchel's inequality).
  Our choice is $\pbf = \pibf$.
  Recall that $K$ is the number of negative samples.
  Then,
  \begin{align}
    \msuploss(\fbf)
    &\ge \E_{c^+, \xbf} \left[ -\fbf(\xbf)^\top\mubf_{c^+} + \sum_{c^- \in \Ycal} \pi_{c^-} \fbf(\xbf)^\top\mubf_{c^-} \right] + \entropy{\pibf} \label{eq:appendix:bound_2} \\
    &\stackrel{\text{(*)}}{=} \E_{c^+} \E_{\xbf, \xbf^+ \sim \Dcal^2_{c^+}} \left[ -\fbf(\xbf)^\top\fbf(\xbf^+) + \fbf(\xbf)^\top \bigg( \E_{c^-} \E_{\xbf^- \sim \Dcal_{c^-}} \left[ \fbf(\xbf^-) \right] \bigg) \right] + \entropy{\pibf} \\
    &= \E_{c^+} \E_{\xbf, \xbf^+} \left[ -\fbf(\xbf)^\top\fbf(\xbf^+) + \frac{1}{K} \sum_{k \in [K]} \E_{c^-_k} \E_{\xbf^-_k} [ \fbf(\xbf)^\top\fbf(\xbf^-_k) ] \right] + \entropy{\pibf} \\
    &= \E_{c^+, \{c^-_k\}_{k}} \E_{\xbf, \xbf^+, \{\xbf^-_k\}_{k}} \left[
      -\frac{1}{K} \sum_{k \in [K]} \left( \fbf(\xbf)^\top\fbf(\xbf^+) - \fbf(\xbf)^\top\fbf(\xbf^-_k) \right)
    \right] + \entropy{\pibf} \\
    &= \E_{c^+, \{c^-_k\}_{k}} \E_{\xbf, \xbf^+, \{\xbf^-_k\}_{k}} \left[
      -\frac{1}{K} \sum_{k \in [K]} \ln \exp(\fbf(\xbf)^\top(\fbf(\xbf^+) - \fbf(\xbf^-_k)))
    \right] + \entropy{\pibf}.
  \end{align}
  Note that the conditional independence is used at (*).
  Here, we can proceed with the Jensen's inequality to lower bound the first term:
  for a non-negative vector $\zbf \in \Rbb_{\geq 0}^N$, the inequality $-N^{-1}\sum_{i \in [N]} \ln z_i \ge -\ln (N^{-1}\sum_{i \in [N]} z_i)$ holds.
  If we set $z_k = \exp(\fbf(\xbf)^\top(\fbf(\xbf^+) - \fbf(\xbf^-_k)))$ for $k \in [K]$,
  \begin{align}
    & \msuploss(\fbf) - \entropy{\pibf} \nonumber \\
    &\ge \E_{\substack{c^+, \{c^-_k\}_k, \\ \xbf, \xbf^+, \{\xbf^-_k\}_k}} \left[
      -\ln \frac{\sum_{k \in [K]} \exp(\fbf(\xbf)^\top(\fbf(\xbf^+) - \fbf(\xbf^-_k)))}{K}
    \right] \\
    &\ge \E_{\substack{c^+, \{c^-_k\}_k, \\ \xbf, \xbf^+, \{\xbf^-_k\}_k}} \left[
      -\ln \frac{ \exp(\fbf(\xbf)^\top(\fbf(\xbf^+) - \fbf(\xbf^+))) + \sum_{k \in [K]} \exp(\fbf(\xbf)^\top(\fbf(\xbf^+) - \fbf(\xbf^-_k)))}{K}
    \right] \\
    &= \E_{\substack{c^+, \{c^-_k\}_k, \\ \xbf, \xbf^+, \{\xbf^-_k\}_k}} \left[
      \ln \frac{ \exp(-\fbf(\xbf)^\top\fbf(\xbf^+)) }{ \exp(-\fbf(\xbf)^\top\fbf(\xbf^+)) + \sum_{k \in [K]} \exp(-\fbf(\xbf)^\top\fbf(\xbf^-_k)) }
    \right] + \ln K.
  \end{align}
  Finally, by using \cref{lemma:cross_entropy_offset},
  \begin{align}
    & \msuploss(\fbf) - \entropy{\pibf} \nonumber \\
    &\ge \E_{c^+, \{c^-_k\}_k} \E_{\xbf, \xbf^+, \{\xbf^-_k\}_k} \left[ -\ln\frac{ \exp(\fbf(\xbf)^\top\fbf(\xbf^+)) }{ \exp(\fbf(\xbf)^\top\fbf(\xbf^+)) + \sum_{k \in [K]} \exp(\fbf(\xbf)^\top\fbf(\xbf^-_k)) } \right] + \ln K  \nonumber \\
    & \quad - 2 \ln \left\{ (K+1) \cosh(L^2) \right\} \\
    &= \nceloss(\fbf) + \ln K - 2 \ln(K + 1) - 2 \ln\cosh(L^2),
  \end{align}
  which concludes the proof.
\end{proof}

\section{Discussion on Relaxing Assumptions}
\label{appendix:relaxation}

In this section, we discuss relaxation of our main results (\cref{theorem:curl_upper_bound,theorem:curl_lower_bound}) from the following perspectives: the conditional independence assumption $\xbf \indep \xbf^+ \mid c^+$, incorporating DA, and the correspondence between the supervised and latent classes $\Ycal = [C]$.
Note that each relaxation is conceptually orthogonal and can be combined together.

\subsection{When Conditionally Independent Assumption Is Violated}
\cref{theorem:curl_upper_bound,theorem:curl_lower_bound} initially rely on the conditional independence, which is used only when we mutually transform the following terms:
\begin{align}
  \E_{c^+, \xbf, \xbf^+} [\fbf(\xbf)^\top\fbf(\xbf^+)] = \E_{c^+, \xbf} [\fbf(\xbf)^\top\mubf_{c^+}].
\end{align}
This operation appears only once in each proof of \cref{theorem:curl_upper_bound,theorem:curl_lower_bound} at (*).
The conditional independence can be removed by the following bound:
\begin{align}
  -2L^2
  \le \E_{c^+, \xbf, \xbf^+} [\fbf(\xbf)^\top\fbf(\xbf^+)] - \E_{c^+, \xbf} [\fbf(\xbf)^\top\mubf_{c^+}]
  = \E_{c^+, \xbf} [\fbf(\xbf)^\top(\fbf(\xbf^+) - \mubf_{c^+})]
  \le 2L^2,
\end{align}
where the inequalities is due to the Cauchy-Schwarz inequality: $|\fbf(\xbf)^\top(\fbf(\xbf^+) - \mubf_{c^+})| \le \|\fbf(\xbf)\| \cdot \|\fbf(\xbf^+) - \mubf_{c^+}\| \le 2L^2$.
An excessive term incurs in the upper/lower bounds by invoking this bound.
\citet{Wang2022ICLR} used the same idea to remove the conditional independence assumption.
Nevertheless, we show \cref{theorem:curl_upper_bound,theorem:curl_lower_bound} with the conditional assumption to focus on the influence of the negative sample size $K$ on the surrogate gap.

\subsection{When Supervised Class Differs from Latent Class}
In unsupervised/self-supervised representation learning, it is often natural not to suppose any relationship between the supervised class set used in a downstream task and the latent classes.
For example, unsupervised data in hand may represent concepts such as `dog', `cat', `taxi', `bus', and `bird', while one expects to classify `animal' or `vehicle' in downstream.
In our main results presented so far, we suppose that the supervised class set $\Ycal$ is the same as the latent classes $[C]$.
Hereafter, we consider several cases where $\Ycal = [C]$ does not hold.

\underline{(Case I) $\Ycal$ is a subset of $[C]$:}
The upper bound (\cref{theorem:curl_upper_bound}) can be extended but the lower bound (\cref{theorem:curl_lower_bound}) cannot for this case.
In the proof of \cref{theorem:curl_upper_bound} (\cref{appendix:proofs}), we showed
\begin{align}
  \nceloss(\fbf) \ge \underbrace{-\E_{c^+, \xbf \sim \Dcal_{c^+}} \ln\frac{\exp(\fbf(\xbf)^\top\mubf_{c^+})}{\sum_{c \in [C]} \exp(\fbf(\xbf)^\top\mubf_{c})}}_\text{cross-entropy loss defined over class set $[C]$} - 2\ln(C\cosh(L^2)) + 2\ln K - \ln K\pi_{(1)}.
\end{align}
The cross-entropy loss defined over latent classes $[C]$ can be lower-bounded as follows:
\begin{align}
  \E_{c^+ \sim \Pbb(c^+ = y \in [C]), \xbf \sim \Dcal_{c^+}} \left[ -\ln\frac{\exp(\fbf(\xbf)^\top\mubf_{c^+})}{\sum_{c \in [C]} \exp(\fbf(\xbf)^\top\mubf_{c})} \right]
  &\ge \E_{c^+ \sim \Pbb(c^+ = y \in \Ycal), \xbf \sim \Dcal_{c^+}} \left[ -\ln\frac{\exp(\fbf(\xbf)^\top\mubf_{c^+})}{\sum_{c \in [C]} \exp(\fbf(\xbf)^\top\mubf_{c})} \right] \\
  &= - \E_{c^+, \xbf}[\fbf(\xbf)^\top\mubf_{c^+}] + \E_{c^+, \xbf} \LSE{\{\fbf(\xbf)^\top\mubf_{c}\}_{c \in [C]}} \\
  &\ge - \E_{c^+, \xbf}[\fbf(\xbf)^\top\mubf_{c^+}] + \E_{c^+, \xbf} \LSE{\{\fbf(\xbf)^\top\mubf_{c}\}_{c \in \Ycal}} \\
  &= \E_{c^+ \sim \Pbb(c^+ = y \in \Ycal), \xbf \sim \Dcal_{c^+}} \left[ -\ln\frac{\exp(\fbf(\xbf)^\top\mubf_{c^+})}{\sum_{c \in \Ycal} \exp(\fbf(\xbf)^\top\mubf_{c})} \right],
\end{align}
where the first inequality is resulted from the non-negativity of the cross-entropy loss and the second inequality uses monotonicity of the log-sum-exp: $\LSE{z_1, \dots, z_N} \ge \LSE{z_1, \dots, z_{N-1}}$.
Hence, the mean supervised loss defined over the supervised class set $\Ycal$ is bounded by the contrastive loss.

\underline{(Case II) $\Ycal$ is a coarse-grained set of $[C]$:}
In this case, we consider $\Ycal$ such as
\begin{align}
  \Ycal = \left\{
    (c_1, \dots, c_{j_1}), (c_{j_1 + 1}, \dots, c_{j_2}), \dots, (c_{j_{J-1} + 1}, \dots, c_{j_J})
    \;\middle|\;
    \begin{aligned}
      & \text{$c_j \in [C]$ for any $j = 1, 2, \dots, c_{j_J}$} \\
      & \text{and $c_{j} \ne c_{j'}$ for any $j, j' = 1, 2, \dots, c_{j_J}$ }
    \end{aligned}
  \right\},
\end{align}
where $j_J = C$.
Intuitively speaking, we split the latent classes $[C]$ into disjoint tuples $(c_j)_{j \in [j_1]}, \dots, (c_j)_{j \in [j_J - j_{J-1}]}$ and regard each tuple in the disjoint set as a coarse-grained class of the original latent classes.
This case aligns with the initial example: `dog', `cat', and `bird' in the latent classes are combined into a single class `animal', while `taxi' and `bus' are combined into `vehicle'.
Both the upper (\cref{theorem:curl_upper_bound}) and lower (\cref{theorem:curl_lower_bound}) bounds can be extended for this case.
We omit the discussion on the upper bound because this is an immediate result by noting the linearity of the expectation and summation over classes.
For the lower bound, after \cref{eq:appendix:bound_2} in the proof (\cref{appendix:proofs}), we need to replace the expectation and summation over the supervised class $\Ycal$ with those over the latent class $[C]$, which is immediate as is the case of the upper bound.

\section{Essential Bounds of Mean Supervised and Contrastive Losses}
\label{appendix:essential_bounds_of_losses}

This section provides a supplementary explanation of the essential lower bounds of the mean supervised and contrastive losses.
The common approaches of CURL applies the normalization on representation, in order to employ the cosine similarity $\frac{\fbf(\xbf)^\top\fbf(\xbf')}{\|\fbf(\xbf)\|_2 \cdot \|\fbf(\xbf')\|_2}$ as the similarity metric.
Then, it is reasonable to assume $\|\fbf(\xbf)\|_2 \le L$ for all $\xbf$ with our data representation $\fbf$.
The normalized representation corresponds to the case $L = 1$.

When we introduce the constraint $\|\fbf(\xbf)\|_2 \le L$, the mean supervised loss and contrastive loss are restricted as well.
As for the mean supervised loss,
\begin{align}
  \msuploss(\fbf)
  &\ge \inf_{\|\fbf'\|_2 \le L} \msuploss(\fbf') \\
  &= \inf_{\|\fbf'\|_2 \le L} \E \left[ \ln\left(1 + \sum_{c \ne y} \exp(\fbf'(\xbf)^\top(\mubf_c - \mubf_y))\right) \right] \\
  &= \ln\left(1 + (C - 1)\exp(-2L^2)\right) \\
  & (\eqdef \msuplossast).
\end{align}
As for the contrastive loss,
\begin{align}
  \nceloss(\fbf)
  &\ge \inf_{\|\fbf'\|_2 \le L} \nceloss(\fbf') \\
  &= \inf_{\|\fbf'\|_2 \le L} \E_{c^+, \{c^-_k\}, \xbf} \E_{\xbf^+, \{\xbf^-_k\}} \left[ \ln\left(1 + \sum_{k \in [K]} \exp(\fbf'(\xbf)^\top(\fbf'(\xbf^-_k) - \fbf'(\xbf^+)))\right) \right] \\
  &\ge \inf_{\|\fbf'\|_2 \le L} \E_{c^+, \{c^-_k\}, \xbf} \left[ \ln\left(1 + \sum_{k \in [K]} \exp(\fbf'(\xbf)^\top(\mubf_{c^-_k} - \mubf_{c^+}))\right) \right] \\
  &= \sum_{m=0}^K \binom{K}{m} \left(\frac{1}{C}\right)^m \left(1 - \frac{1}{C}\right)^{K-m} \ln\{1 + m + (K - m)\exp(-2L^2)\} \\
  & (\eqdef \ncelossast),
\end{align}
where the Jensen's inequality is applied in the second inequality.

\section{Discussion of Existing Surrogate Bounds}
\label{appendix:existing_bounds}

In this section, we describe the existing surrogate bounds in details to make them comparable with our main results.
Then, we further discuss the detailed comparison between our theory and existing works.
Before the discussion, we need to introduce the \emph{sub-class} loss (of the mean classifier), which is the supervised classification loss over a subset of classes:
\begin{align}
  \subloss(\fbf, T) \eqdef \E_{\xbf, y} \left[ -\ln \frac{\exp(\mubf_y^\top\fbf(\xbf))}{\sum_{c \in T} \exp(\mubf_c^\top\fbf(\xbf))} \right],
\end{align}
where $T \subseteq [C]$ is a subset of classes and $y$ is drawn from the subset of $\pibf$ with respect to $T$.

\paragraph{\citeauthor{Arora2019ICML}'s bound.}
We introduce additional notation that \citet{Arora2019ICML} use.
For a subset of classes $T$,
\begin{itemize}
  \item $Q \subseteq [C]$ is the set of distinct classes in $c^+, c^-_1, \dots, c^-_K$
  \item $I^+ \eqdef \{k \in [K] \mid c^-_k = c^+ \}$
  \item $\col \eqdef \sum_{k \in [K]} \indicator{c^+ = c^-_k} = |I^+|$
  \item $\rho_\mathrm{max}(T) \eqdef \max_{c \in T} \pi_c$
  \item $\rho_\mathrm{min}^+(T) \eqdef \min_{c \in T} \Pbb_{c^+, \{c^-_k\}_k \sim \pibf^{K+1}} (c^+ = c \mid Q = T, I^+ = \emptyset)$
  \item $\tau_K \eqdef \Pbb(I^+ \ne \emptyset)$
\end{itemize}
\citet{Arora2019ICML} prove a finite-sample surrogate bound in Theorem~B.1.
In its proof, Eq.~(26) is a surrogate bound established for a fixed $\fbf$.
For the comparison, we focus on their Eq.~(26):
\begin{align}
  (1 - \tau_K) \E_{T \sim \pibf^{K+1}} & \left[ \frac{\rho_\mathrm{min}^+(T)}{\rho_\mathrm{max}(T)} \subloss(\fbf, T) \right] \nonumber \\
  &\le \nceloss(\fbf) - \tau_K \E_{c^+, \{c^-_k\}_k \sim \pibf^{K+1}} \left[ \ln(\col + 1) \middle| I^+ \ne \emptyset \right].
\end{align}
We split the expectation term in the left-hand side as follows.
\begin{align}
  \E & \left[ \frac{\rho_\mathrm{min}^+(T)}{\rho_\mathrm{max}(T)} \subloss(\fbf, T) \right] \nonumber \\
  &= \underbrace{\Pbb(\text{$T$ covers $[C]$})}_{= v_{K+1}} \cdot \E \left[ \frac{\rho_\mathrm{min}^+(T)}{\rho_\mathrm{max}(T)} \subloss(\fbf, T) \middle| \text{$T$ covers $[C]$} \right] \nonumber \\
    & \qquad + \Pbb(\text{$T$ does not cover $[C]$}) \cdot \E \left[ \frac{\rho_\mathrm{min}^+(T)}{\rho_\mathrm{max}(T)} \subloss(\fbf, T) \middle| \text{$T$ does not cover $[C]$} \right] \\
  &\ge v_{K+1} \frac{\rho_\mathrm{min}^+([C])}{\rho_\mathrm{max}([C])} \underbrace{\subloss(\fbf, [C])}_{= \msuploss(\fbf)}.
\end{align}
Under the uniform class prior assumption ($\pibf = \nicefrac{1}{C} \cdot \onebf$), $\rho_\mathrm{max}([C]) = \nicefrac{1}{C}$,
and we can pick any class $c_0 \in [C]$ by the symmetry and $\rho_\mathrm{min}^+([C]) = \Pbb(c^+ = c_0 \mid Q = [C], I^+ = \emptyset) = \nicefrac{1}{C}$.
In addition,
\begin{align}
  \tau_K \E_{c^+, \{c^-_k\}_k \sim \pibf^{K+1}} \left[ \ln(\col + 1) \middle| I^+ \ne \emptyset \right]
  &= \E \left[ \ln(\col + 1) \right] - (1 - \tau_K) \E \left[ \ln(\col + 1) \middle| I^+ = \emptyset \right] \\
  &= \E \left[ \ln(\col + 1) \right].
\end{align}
As a result, we obtain the following simplified expression in \cref{table:existing_bounds}:
\begin{align}
  \msuploss(\fbf) \le \frac{1}{(1 - \tau_K)v_{K+1}} \left\{ \nceloss(\fbf) - \E[\ln(\col + 1)] \right\}.
\end{align}

\paragraph{\citeauthor{Nozawa2021NeurIPS}'s bound.}
The surrogate bound provided by \citet[Theorem~8]{Nozawa2021NeurIPS} involves a factor resulting from DA and self-supervised learning setting.
By dropping this (negative) factor, the surrogate bound is
\begin{align}
  \nceloss(\fbf)
  &\ge \frac{1}{2} \left\{ v_{K+1} \msuploss(\fbf) + (1-v_{K+1}) \E_{T \sim \pibf^{K+1}}[\subloss(\fbf, T)] + \E[\ln(\col + 1)] \right\} \\
  &\ge \frac{1}{2} \left\{ v_{K+1} \msuploss(\fbf) + \E[\ln(\col + 1)] \right\},
\end{align}
resulting in the bound in \cref{table:existing_bounds}.
The sub-class loss may be safely dropped because it has the coefficient $1 - v_{K+1}$, which is expected to be exponentially small in $K$.

\paragraph{\citeauthor{Ash2022AISTATS}'s bound.}
\citet[Theorem~5]{Ash2022AISTATS} provides the following surrogate bound
\begin{align}
  \msuploss(\fbf) \le \frac{2 \left\lceil \frac{2(1 - \pi_{(-1)})H_{C-1}}{K\pi_{(-1)}} \right\rceil}{(1 - \pi_{(1)})^K} \left\{ \nceloss(\fbf) - \tau_K \E_{c^+, \{c^-_k\}_k \sim \pibf^{K+1}} \left[ \ln(\col + 1) \middle| I^+ \ne \emptyset \right] \right\}.
\end{align}
By substituting $\pibf = \nicefrac{1}{C} \cdot \onebf$ and $\tau_K \E[\ln(\col + 1) \mid I^+ \ne \emptyset] = \E[\ln(\col + 1)]$, the bound in \cref{table:existing_bounds} is obtained.

\paragraph{Detailed comparisons.}
As we stated in \cref{section:comparison_with_existing} of the main text, only our bound agrees well with the experimental fact that the larger $K$ is better for \textit{all} $K$ regions:
\begin{itemize}
  \item \citet{Arora2019ICML}: Large $K$ degrades the performance because of the label collision.
  \item \citet{Nozawa2021NeurIPS}: Large $K$ improves the performance for $K > C$.
  \item \citet{Ash2022AISTATS}: The optimal $K$ exists by the collision-coverage trade-off.
\end{itemize}
Even though the claim by \citet{Nozawa2021NeurIPS} is similar to ours, we discovered a different underlying mechanism to support this idea, which leads to better explainability of empirical facts.

The proof of \citet{Nozawa2021NeurIPS} is based on the idea of label coverage: The more negative samples we draw (larger $K$), the more likely the negative samples can cover all class labels.
The upper bound based on this idea is only activated when $K>C$ because label coverage is impossible with $K \le C$.
This inability contradicts the real experiments including \citet{Chen2021NeurIPS,Tomasev2022arXiv}, which showed that CURL exhibits reasonable performance even with small $K$.

Our proof leverages the idea that $\nceloss$ and $\msuploss$ have the similar log-sum-exp functional forms.
This similarity casts $\nceloss$ as a surrogate objective of $\msuploss$ and its surrogate gap is reduced with larger $K$.
Even with small $K$, the upper bound of $\msuploss$ is loose but not prohibitively large thereby the surrogate bound of $\msuploss$ is still valid.
Our theoretical claim reveals that the surrogate gap improves in $O(K^{-1})$ for \text{all} $K$ regions, which is in good agreement with the real experiments.
Eventually, our theory provides practical feedback such that one may reduce $K$ (even smaller than $C$) to trade off the downstream performance with the computational cost.

\section{Relationship to Mutual Information (MI) Estimation}
\label{appendix:mutual_information_estimation}
The contrastive loss $\nceloss$ we studied in this paper is also known as the InfoNCE loss \citep{Oord2018arXiv}, which is known to be deeply related to the multi-sample estimation of mutual information (MI) \citep{Oord2018arXiv,Poole2019ICML,Song2020ICLR}.
Although the multi-sample estimators have high bias and low variance compared to variational estimators in general~\citep{Poole2019ICML,Song2020ICLR,Guo2021arXiv}, the quantitative analysis of the bias-variance trade-off of the multi-sample MI estimators has yet to be clearly known.
\citet{Tian2020NeurIPS} and \citet{Tschannen2020ICLR} experimentally showed that maximizing \textit{tighter} MI bound does not necessarily lead to good representation; there is no guarantee that the model can achieve higher MI by the tighter bound.

Recently, the theoretical limitations of sample-based MI estimation have been analyzed \citep{Gao2015AISTATS,McAllester2020AISTATS}.
These studies revealed that a particular type of sample-based estimator of MI \citep{Gao2015AISTATS} or its lower bound \citep{McAllester2020AISTATS} can be upper bounded by $O(\ln N)$ for the number of samples $N$.
In this section, we discuss the implications of these limitations in the CURL setting.

Given two random variables $X$ and $Y$, suppose that we have $K + 1$ randomly drawn pairs $\{(x_i, y_i)\}_{i=1}^{K+1}$ from these random variables such that for all $(i, j)$, $(x_i, y_j)$ can be regarded as a positive pair when $i = j$, and otherwise can be regarded as a negative pair.
\citet[Equation 10]{Poole2019ICML} derived the following lower bound for MI:
\begin{equation}
  I(X; Y) \ge I_{\mathrm{NCE}}^{K+1} \coloneqq \mathbb{E} \left[ \frac{1}{K+1} \sum_{i=1}^{K+1} \ln \frac{\exp(s(x_i, y_i))}{ \frac{1}{K+1} \sum_{j=1}^{K+1} \exp(s(x_i, y_j)) } \right],
\end{equation}
where $I(X; Y)$ is the MI between $X$ and $Y$, and $s(x, y)$ is a critic function.
This lower bound estimator can be rewritten using $\nceloss$ as follows:
\begin{align}
  I_{\mathrm{NCE}}^{K+1} &\coloneqq \mathbb{E} \left[ \frac{1}{K+1} \sum_{i=1}^{K+1} \ln \frac{ \exp(s(x_i, y_i))}{ \frac{1}{K+1} \sum_{j=1}^{K+1} \exp(s(x_i, y_j)) } \right] \\
  &= \mathbb{E} \left[ \frac{1}{K+1} \sum_{i=1}^{K+1} \ln \frac{\exp ({s(x_i, y_i)})}{ \sum_{j=1}^{K+1} \exp ({s(x_i, y_j)}) } \right] + \ln(K+1) \\
  &= \mathbb{E} \left[ \frac{1}{K+1} \sum_{i=1}^{K+1} \ln \frac{\exp ({s(x_i, y_i)})}{ \exp ({s(x_i, y_i)}) + \sum_{j \neq i} \exp(s(x_i, y_j)) } \right] + \ln(K+1) \\
  &= \mathbb{E} \left[ \ln \frac{ \exp(s(x, x^+))}{  \exp(s(x, x^+)) + \sum_{j = 1}^K \exp(s(x, x^-_j)) } \right] + \ln(K+1) \\
  &= - \nceloss(\fbf) + \ln (K + 1).
\end{align}
The first equality is obtained by putting the constant in the denominator outside.
The third equality comes by replacing the notation $(x_i, y_i)$ with $(x, x^+)$ under the assumption that all $(x_i, y_i)$ come from the same iid distribution.
By setting $s(x, y) \coloneqq \fbf(x)^\top \fbf(y)$, we obtain the last equation.

Here, \citet{McAllester2020AISTATS} gave the following theorem for the sample-based estimator of the lower bound on MI.
\begin{theorem}[\citet{McAllester2020AISTATS} Theorem 1.1, informal]
  Let $\widehat{I}^N$ be any mapping from $N$ samples of $(X, Y)$ to $\mathbb{R}$ that satisfies
  \begin{equation}
    I(X; Y) \ge \widehat{I}^N(\{(x_i, y_i)\}_{i=1}^N)
  \end{equation}
  in high probability, then the following relationship holds in high probability:
  \begin{equation}
    \widehat{I}^N(\{(x_i, y_i)\}_{i=1}^N) \le 2 \ln N + 5.
  \end{equation}
\end{theorem}

Since $I_{\mathrm{NCE}}^{K+1}$ satisfies the condition for $\widehat{I}^{K+1}$, we now have the following:
\begin{equation}
  - \nceloss(\fbf) + \ln (K + 1) \le 2 \ln (K + 1) + 5 \implies \nceloss(\fbf) \ge - \ln(K + 1) - 5.
\end{equation}
However, the right-hand statement always holds by the construction of $\nceloss$ for all $\fbf$ ($\forall \fbf, \nceloss(\fbf) \ge 0$).
In other words, in the case of the CURL setting, \citet{McAllester2020AISTATS}'s theorem does not restrict $\nceloss$, which means that the large $K$ effect investigated in our paper comes from a completely different mechanism from the above theorem.
While the theoretical studies on MI aim to guarantee for multi-sample estimation of ground-truth MI, a series of CURL studies, including ours, differ in that we aim to derive a surrogate gap bound between two different losses, namely supervised loss and contrastive loss.
Specifically, the above derivation does not address supervised loss in that both the left-hand ($I(X; Y)$ or $I_{\mathrm{NCE}}^{K+1}$) and the right-hand ($\nceloss(\fbf)$) quantities represent the amount of information between different views rather than MI between the view and its label.
The existing studies on sample-based MI estimation are worthwhile in the sense that these works revealed the $O(\ln N)$ effect on the non-trivial estimators such as $k$-NN based estimator \citep{Gao2015AISTATS} or \textit{any} kind of lower bound estimator \citep{McAllester2020AISTATS}.

\section{Experimental details}
\label{appendix:experimental_details}

\subsection{Synthetic Dataset}
We used Adam~\citep{Kingma2015ICLR} optimizer with the weight decay of coefficient $\num{0.01}$ to all parameters.
The mini-batch size was set to $B = \num{1024}$ and the number of epochs was $\num{300}$.
The learning rate was set to $\num{0.01}$ with \texttt{ReduceLROnPlateau} scheduler (patience: $\num{10}$ epochs) provided by PyTorch~\citep{Paszke2019NeurIPS}.

\subsection{CIFAR-10/100}
We treated $\num{10}\%$ training samples as a validation dataset by sampling class uniformly.
We used the original test dataset for testing.
We used the same data-augmentation as in the CIFAR-10 experiment by~\citet{Chen2020ICML} during contrastive learning and linear supervised training of the linear classifier.

As a feature extractor $\fbf$, we modified the ResNet-18~\citep{He2016CVPR} by following the convention of self-supervised representation learning~\citep[B.9]{Chen2020ICML};
replacement of the first convolutional layer with a smaller one, removal of the first max-pooling layer, and replacement of the final fully-connected layer with a nonlinear projection head whose dimensional is $\num{32}$.%
\footnote{
  Unlike the reported results by \citet{Chen2021NeurIPS}, smaller dimensionality, i.e., $\num{32}$ gives better downstream accuracy on CIFAR-100 than $\num{64}$ or $\num{128}$.
  This difference might come from the differences in the loss function and positive pair's generation process.
}

Since we need to enlarge the negative samples size $K$ that depends on the size of mini-batches,
we followed a large mini-batch training setting used in recent self-supervised learning~\citep{Chen2020ICML,Caron2020NeurIPS}.
We used LARC~\citep{You2017techrep} optimizer wrapping the momentum SGD, whose momentum term was $\num{0.9}$.
We applied weights decay of coefficient $\num{e-4}$ to all parameters except for all bias terms and batch norm's parameters.
The base learning rate was initialized at $\text{lr} \times \sqrt{B}$, where $\text{lr} \in \{\num{2}, \num{4}, \num{6}\} \times \nicefrac{1}{64}$ and mini-batch size $B=\num{1024}$ inspired by SimCLR's squared learning rate scaling.
As a learning rate scheduler for each iteration, we used linear warmup during the first $\num{10}$ epochs and cosine annealing without restart~\citep{Loshchilov2017ICLR} during the rest epochs.
The number of epochs was $\num{2000}$.

We implemented our experimental code by using PyTorch \citep{Paszke2019NeurIPS}'s distributed data-parallel training~\citep{Li2018VLDB} on $\num{8}$ NVIDIA A100 GPUs provided by the internal cluster.
Therefore we replaced the all batch normalization layer with \texttt{SyncBatchNorm} module provided by PyTorch.\footnote{See \citet[Sec. 6.2]{Wu2021arXiv} for more detailed discussion of this replacement for contrastive learning.}
To accelerate contrastive learning, we used automatic mixed-precision training provided by PyTorch.

\subsection{Wiki-3029}
Wiki-3029 contains $\num{3029}$ English Wikipedia article pages.
Each page consists of $\num{200}$ sentences.
Since the dataset does not have the explicit train/validation/test splits, we split the dataset into $\num{70}\%/\num{10}\%/\num{20}\%$ train/validation/test datasets, respectively.
As a pre-processing, we tokenized the dataset using torchtext's \texttt{basic\_english} tokenizer.
After tokenization, we removed the tokens whose frequency is less than $\num{5}$ in the training dataset.
We did not use DA.

We used fasttext~\citep{Joulin2017EACL}'s based feature extractor.\footnote{\citet{Arora2019ICML} uses GRU-based feature encoder with frozen word embeddings of GloVe~\citep{Pennington2014EMNLP} trained on commonCrawl.}
In our preliminary experiments, only using a word embedding layer and average pooling among words perform better than either additional linear or nonlinear projection heads.
A similar model to ours is also used in~\citet{Ash2022AISTATS}.
The dimensionality of the word embedding layer was $\num{256}$.

We mainly followed the same optimization setting as our CIFAR-10/100 experiments.
We note that the mini-batch size $B=\num{2048}$; the initial learning rate $\text{lr}$ was selected in $\{\num{1}, \num{2}, \num{3}, \num{4}\} \times \nicefrac{1}{40}$; no weights decay; the number of epochs was $\num{90}$; and perform linear warmup during the first $\num{3}$ epochs.
When we decrease $C$, the number of epochs is multiplied by $\nicefrac{\num{3000}}{C}$ for simplicity.\footnote{We found the contrastive learning did not yield good feature representations for a downstream task without this longer training.}

\subsection{Contrastive Learning}
\label{sec:contrastive-learning-details}
By following the data generation process in contrastive representation learning and existing work~\citep{Arora2019ICML,Ash2022AISTATS},
we treated the supervised classes $\Ycal$ as latent classes $[C]$.
After obtaining training/validation/test datasets as described above,
we carefully constructed positive pairs for contrastive learning \textit{before} training%
\footnote{
  We can create the labeled dataset, especially with non-overlapped latent classes,
  if we draw positive samples at each iteration or epoch during optimization using stochastic gradient descent.
}
as follows;
We treated each sample in the training data as an anchor sample.
We drew a different sample from the same latent class of each anchor sample as a positive sample in the training dataset.
For negative samples, we drew $K$ negative samples from other samples in the same mini-batch by following the convention of self-supervised representation learning such as SimCLR~\citep{Chen2020ICML}.
Since \citet{Chen2020ICML} used all other samples as negative samples, the negative samples size and the size of mini-batches depend on each other: $K=2B-2$.
To relax the effect of the difference of the mini-batch size when we change $K$,
we drew $K$ samples without replacement from $2B-2$ inspired by~\citet{Ash2022AISTATS}.
In this sampling, we guaranteed to draw at most one sample from each positive pair because we are concerned about the relation between the number of latent classes and $K$.
We did not use validation and test datasets during contrastive representation learning.

\subsection{Mean and Linear Classifiers' Evaluation}
\label{sec:mean-linear-eval-details}
For evaluation, we reported the test accuracy values of mean and linear classifiers.
For a linear classifier, we used Nesterov's momentum SGD, whose momentum coefficient was $\num{0.9}$ without weight decay.
We set the mini-batch size $B=\num{256}$ and $B=\num{512}$ for CIFAR-10/100 and Wiki-3029, respectively.
We used cosine annealing without restart as a learning rate scheduler for each iteration.  %
We set $\num{100}$ and $\num{30}$ epochs for CIFAR-10/100 and Wiki-3029 datasets, respectively.
For CIFAR-10/100, we set learning rate as $\num{0.03}$.
For Wiki-3029, we searched the learning rate in $\{\num{0.5}, \num{1}, \num{5}, \num{10}, \num{50}\} \times \nicefrac{1}{10^3}$.
The learning rate was scaled by using squared learning rate scaling.
For linear evaluation of CIFAR-10/100, we used PyTorch's distributed data-parallel training.
We calculated the test accuracy by using the best combination of the contrastive model and the hyper-parameter of a linear classifier that maximizes the validation accuracy.
We repeated contrastive learning and downstream task's evaluation three times with different random seeds and reported the averaged values.

\subsection{Details of \texorpdfstring{\cref{figure:upper_bound_comparison_cifar10}}{Figure~\ref{figure:upper_bound_comparison_cifar10}}}
\label{appendix:experimental_details_figure4}
Before computing the upper bounds and supervised loss, we normalized feature representations $\fbf(\xbf)$ learned in \cref{sec:contrastive-learning-details} to ensure $L=1$, which is the upper bound of $\| \fbf(\xbf) \|_2, \forall \xbf$.
For each random seed and the number of negative samples $K$, we selected learned feature encoder $\fbf$ that got the highest validation mean supervised accuracy in different learning rates of the optimizer of the contrastive learning.
Then we calculated the test supervised loss value by using the selected contrastive models.

Using the same feature encoder with $L_2$ normalization, we calculated the contrastive loss on the test dataset.
To do so, we created positive pairs by the same procedure on the test dataset as described in~\cref{sec:contrastive-learning-details}.
Negative samples were also drawn from the other samples in the mini-batches as the contrastive learning step described in \cref{sec:contrastive-learning-details}.
To calculate the contrastive loss, we used the same batch size as the contrastive learning step and only one epoch.
Since this contrastive loss calculation was stochastic due to the sampling of positive and negative samples, we repeated the contrastive loss calculation $\num{25}$ times and averaged them to create plot~\cref{figure:upper_bound_comparison_cifar10}.
Note that we used the theoretical values of $\tau_K, v_{K+1}, \E\ln(\col + 1)$ that are shown in the existing upper bounds on \cref{table:existing_bounds} rather than the simulated values.

\begin{figure}[t]
  \centering
  \begin{subfigure}[b]{0.4\textwidth}
    \includegraphics[width=\textwidth]{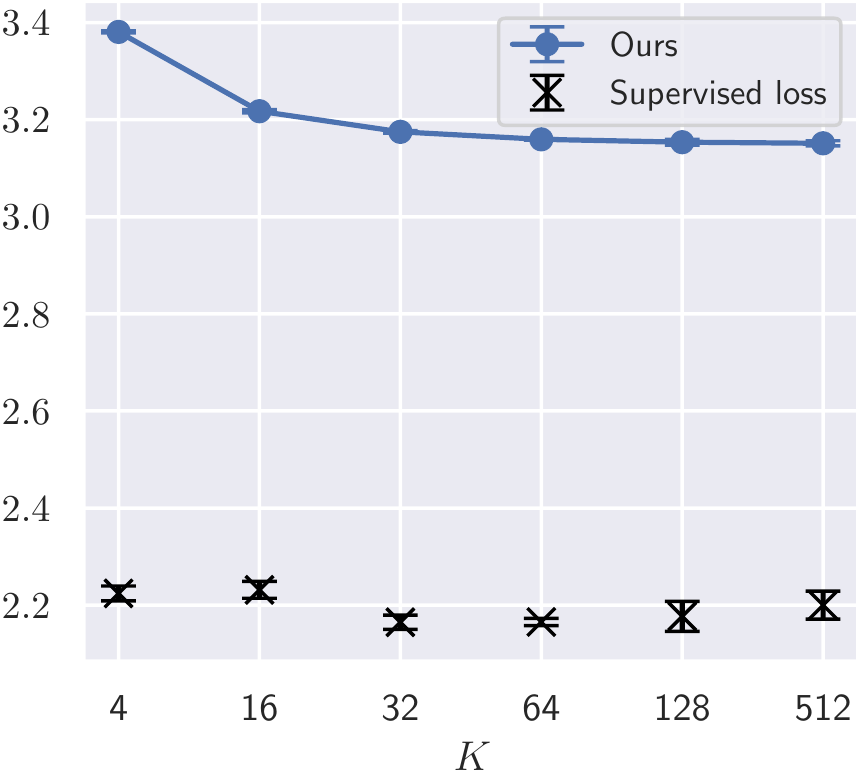}
    \caption{CIFAR-10.}
    \label{figure:cifar10-closer-look}
  \end{subfigure}
  \begin{subfigure}[b]{0.4\textwidth}
    \includegraphics[width=\textwidth]{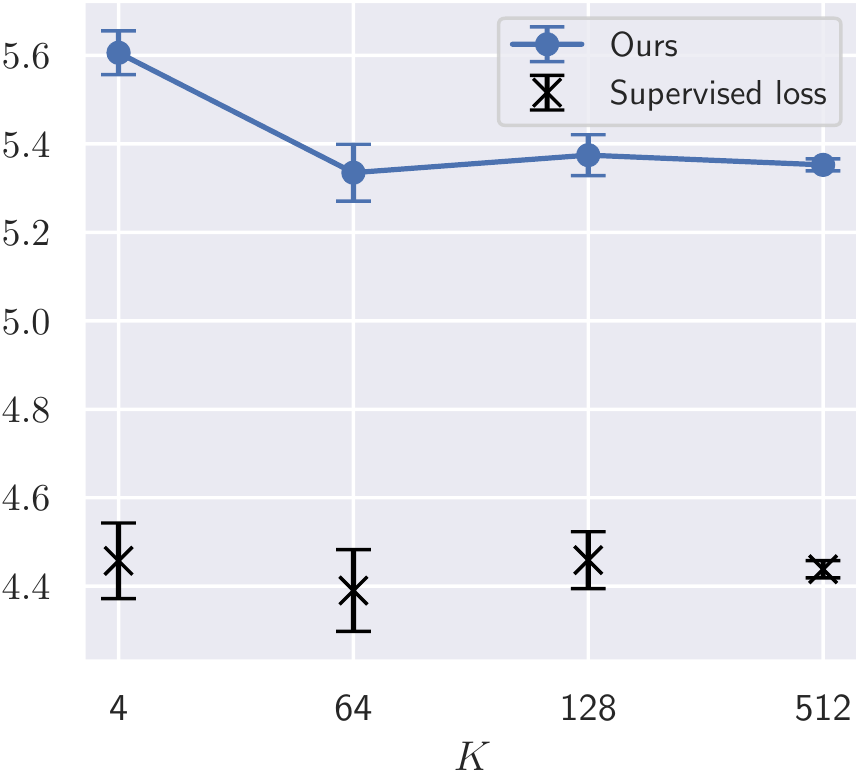}
    \caption{CIFAR-100.}
    \label{figure:cifar100-closer-look}
  \end{subfigure}
  \caption{
  Enlarged \cref{figure:upper_bound_comparison_cifar10} for the detailed comparison between the proposed bound and the supervised loss on CIFAR-10/100 datasets.
  All value is an averaged value among three runs with a different random seed.
  Error bar indicates the standard deviation.
  }
  \label{figure:cifar-closer-look}
\end{figure}

\Cref{figure:cifar-closer-look} shows the enlarged version of \cref{figure:upper_bound_comparison_cifar10} and the same plot using CIFAR-100.
This figure focuses on the detailed comparison between the test datasets' empirical supervised loss values and theoretical bounds.
For both CIFAR-10/100 datasets, there were almost no changes in the supervised loss as $K$ varied, and the losses were slightly larger in the region where $K$ was small.
These results are consistent with the theoretical estimation of the upper bounds (solid lines).

\subsection{Details of \texorpdfstring{\cref{figure:cifar-epoch}}{Figure~\ref{figure:cifar-epoch}}}
During minimization of the contrastive loss to learn $\fbf$ in \cref{sec:contrastive-learning-details}, we saved the model's weight at every $\num{200}$ epochs.
We reported the test mean supervised accuracy using $\fbf$ that maximized validation accuracy among different learning rate values.

\subsection{Additionally Used Libraries}
In our experiments,
we also used scikit-learn~\citep{JMLR:v12:pedregosa11a} for train/val/test data splits.
We created all plots by using matplotlib~\citep{Hunter2007matplotlib} and seaborn~\citep{Waskom2021seaborn} via pandas~\citep{pandas2020} except for \cref{figure:feasible_region}.
We managed our experiments' configuration using hydra~\citep{Yadan2019Hydra} and experimental results using Weights \& Biases~\citep{wandb}.
For effective parallelized execution of our experimental codes, we use GNU Parallel~\citep{tange_2021_5523272}.

\end{document}